\documentclass{article} 
\usepackage[usenames,dvipsnames]{xcolor}
\usepackage{nips15submit_e,times}
\usepackage{hyperref}
\usepackage{url}

\usepackage{algorithm}
\usepackage{algorithmic}
\usepackage[capitalize]{cleveref}
\usepackage{float}
\usepackage{amsmath}
\usepackage{amsthm}
\usepackage{amsfonts}
\usepackage{graphicx}
\usepackage[numbers]{natbib}
\usepackage{xspace}
\allowdisplaybreaks

\usepackage{color}

\usepackage[disable,backgroundcolor = White,textwidth=\marginparwidth]{todonotes}

\newcommand{\todoa}[2][]{\xspace\todo[size=\tiny,color=Blue!20,#1]{A: #2}}
\newcommand{\todoy}[2][]{\xspace\todo[color=Cerulean!20,size=\tiny,#1]{Y: #2}}
\newtheorem{lemma}{Lemma}
\newtheorem{theorem}[lemma]{Theorem}
\newtheorem{cor}[lemma]{Corollary}
\newtheorem{proposition}[lemma]{Proposition}
\theoremstyle{definition}

\newtheorem{remark}[lemma]{Remark}

\newcommand{\iset}[1]{\left[#1\right]}
\newcommand{\cset}[2]{\left\{#1\,:\,#2\right\}}

\newcommand{\rar}{\rightarrow}
\newcommand{\lar}{\leftarrow}

\DeclareMathOperator*{\argmin}{argmin}
\DeclareMathOperator*{\argmax}{argmax}
\DeclareMathOperator*{\KL}{KL}
\newcommand{\inprod}[2]{\left\langle #1 , #2 \right\rangle}
\newcommand{\norm}[1]{\left\Vert #1 \right\Vert}

\newcommand{\N}{\mathbb{N}}
\newcommand{\Normal}{\mathcal{N}}
\newcommand{\R}{\mathbb{R}}

\newcommand{\gap}{d}
\newcommand{\e}{\epsilon}
\newcommand{\F}{\mathcal{F}}
\newcommand{\E}[1]{\mathbb{E}\left[#1\right]}
\newcommand{\I}[1]{\mathbb{I}\left\{#1\right\}}
\newcommand{\Prb}[1]{\Pr\left( #1 \right)}
\newcommand{\htheta}{\hat{\theta}}

\newcommand{\simplex}{\Delta}

\newcommand{\lc}{\lceil}
\newcommand{\rc}{\rceil}

\newcommand{\Str}{\mathcal{S}}
\newcommand{\Wea}{\mathcal{W}}
\newcommand{\addeq}{\addtocounter{equation}{1}\tag{\theequation}}
\newcommand{\dom}{\rho}
\newcommand{\domlp}{\dom_{\text{LP}}}
\newcommand{\indp}{\kappa}
\newcommand{\wtTheta}{\widetilde{\Theta}}

\title{Online Learning with Gaussian Payoffs and Side Observations}

\author{
Yifan Wu,\  Csaba Szepesv\'ari\\
Dept. of Computing Science\\
University of Alberta\\
\texttt{\{ywu12,szepesva\}@ualberta.ca} \\
\And
Andr\'as Gy\"orgy \\
Dept. of Electrical and Electronic Engineering \\
Imperial College London \\
\texttt{a.gyorgy@imperial.ac.uk} \\
}

\nipsfinalcopy 

\begin{document}

\maketitle

\begin{abstract}
We consider a sequential learning problem with Gaussian payoffs and side information: after selecting an action $i$, the learner
receives information about the payoff of every action $j$ in the form of Gaussian observations whose mean is the same as the mean payoff, but the variance depends on the pair $(i,j)$ (and may be infinite). The setup allows a more refined information transfer from one action to another than previous partial monitoring setups, including the recently introduced graph-structured feedback case. For the first time in the literature, we provide non-asymptotic problem-dependent lower bounds on the regret of any algorithm, which recover existing asymptotic problem-dependent lower bounds and finite-time minimax lower bounds available in the literature. We also provide algorithms that achieve the problem-dependent lower bound (up to some universal constant factor) or the minimax lower bounds (up to  logarithmic factors). 
\end{abstract}

\section{Introduction}

Online learning in stochastic environments is a sequential decision problem where in each time step a learner chooses an action from a given finite set, observes some random feedback and receives a random payoff. Several feedback models have been considered in the literature: The simplest is the full information case where the learner observes the payoff of all possible actions at the end of every round. A popular setup is the case of bandit feedback, where the learner only observes its own payoff and receives no information about the payoff of other actions \cite{BuCB12}. 
Recently, several papers considered a more refined setup, called graph-structured feedback, that interpolates between the full-information and the bandit case: here the feedback structure is described by a (directed) graph, and choosing an action reveals the payoff of all actions that are connected to the selected one, including the chosen action itself. This problem, motivated for example by social networks, has been studied extensively in both the adversarial  \cite{MaSha11,AlCeGe13,KoNeVa14,AlCeDe15} and the stochastic cases \cite{CaKvLe12,BuErSh14}. 
However, most algorithms presented heavily depend on the self-observability assumption (that is, that the payoff of the selected action can be observed). Removing this self-loop assumption leads to the so-called partial monitoring case \cite{AlCeDe15}.
In the absolutely general partial monitoring setup the learner receives some general feedback that depends on its choice (and the environment), with some arbitrary (but known) dependence \cite{CeLu06,BaFoPaRaSze14}. While the partial monitoring setup covers all other problems, its analysis has concentrated on the finite case where both the set of actions and the set of feedback signals are finite \cite{CeLu06,BaFoPaRaSze14}, which is in contrast to the standard full information and bandit settings where the feedback is typically assumed to be real-valued.
The only exception to this case is the work of \cite{AlCeDe15}, which considers graph-structured feedback without the self-loop assumption.

In this paper we consider a generalization of the graph-structured feedback model that can also be viewed as a general partial monitoring model with real-valued feedback. We assume that selecting an action $i$ the learner can observe a random variable $X_{ij}$ for each action $j$ whose mean is the same as the payoff of $j$, but its variance $\sigma^2_{ij}$ depends on the pair $(i,j)$. For simplicity, throughout the paper we assume that all the payoffs and the $X_{ij}$ are Gaussian. While in the graph-structured feedback case one either has observation on an action or not, but the observation always gives the same amount of information, our model is more refined: Depending on the value of $\sigma_{ij}$, the information can be of different quality. For example, if $\sigma^2_{ij} = \infty$, trying action $i$ gives no information about action $j$. In general, for any $\sigma_{ij}^2<\infty$, the value of the information depends on the time horizon $T$ of the problem: when $\sigma^2_{ij}$ is large relative to $1/\sqrt{T}$ (and the payoff differences of the actions) essentially no information is received, while a small variance results in useful observations. 

After defining the problem formally in Section~\ref{sec:pre}, we provide non-asymptotic problem-dependent lower bounds in Section~\ref{sec:lower}, which depend on the distribution of the observations through their mean payoffs and variances. To our knowledge, these are the first such bounds presented for any stochastic partial monitoring problem beyond the full-information setting: previous work either presented asymptotic problem-dependent  lower bounds (e.g., \cite{GrLa97,BuErSh14}), or finite-time minimax bounds (e.g., \cite{BaFoPaRaSze14,AlCeGe13,AlCeDe15}). Our bounds can recover all previous bounds up to some universal constant factors not depending on the problem. In Section~\ref{sec:algs}, we present two algorithms with finite-time performance guarantees for the case of graph-structured feedback without the self-observability assumption. While due to their complicated forms it is hard to compare our finite-time upper and lower bounds, we show that our first algorithm achieves the asymptotic problem-dependent lower bound up to problem-independent multiplicative factors. Regarding the minimax regret, the hardness ($\wtTheta(T^{1/2})$ or $\wtTheta(T^{2/3})$ regret) of partial monitoring problems is characterized by their global/local observability property \cite{BaFoPaRaSze14} or, in case of the graph-structured feedback model, by their strong/weak observability property \cite{AlCeDe15}. In the same section we present another algorithm that achieves the minimax regret (up to logarithmic factors) under both strong and weak observability, and achieves an $O(\log^{3/2} T )$ problem-dependent regret. Earlier results for the stochastic graph-structured feedback problems \cite{CaKvLe12,BuErSh14} provided only asymptotic problem-dependent lower bounds and performance bounds that did not match the asymptotic lower bounds or the minimax rate up to constant factors. Finally, we draw conclusions and consider some interesting future directions in Section~\ref{sec:conc}. 
Due to space constraints, all proofs are deferred to the appendix.

\section{Problem Formulation} 
\label{sec:pre}

Formally, we consider an online learning problem with \emph{Gaussian payoffs and side observations}:
Suppose a learner has to choose from $K$ actions in every round. When choosing an action, the learner receives a random payoff and also some side observations corresponding to other actions. More precisely, each action $i \in \iset{K}=\{1,\ldots,K\}$ is associated with some parameter $\theta_i$, and the payoff $Y_{t,i}$ to action $i$ in round $t$ is normally distributed random variable with mean $\theta_i$ and variance $\sigma^2_{ii}$, while the learner observes a $K$-dimensional Gaussian random vector $X_{t,i}$ whose $j$th coordinate is a normal random variable with mean $\theta_j$ and variance $\sigma_{ij}^2$ (we assume $\sigma_{ij} \ge 0$) and the coordinates of $X_{t,i}$ are independent of each other. We assume the following: (i) the random variables $(X_t,Y_t)_t$ are independent for all $t$; (ii) the parameter vector $\theta$ is unknown to the learner but it knows the variance matrix $\Sigma=(\sigma^2_{ij})_{i,j\in \iset{K}}$ in advance; (iii) $\theta \in [0,D]^K$ for some $D>0$\todoa{Is this really needed?}; (iv) $\min_{i \in \iset{K}} \sigma_{ij} \le \sigma<\infty$ for all $j\in\iset{K}$, that is, the expected payoff of each action can be observed. 

The goal of the learner is to maximize its payoff or, in other words, minimize the expected regret
\[
R_T= T\max_{i\in\iset{K} }\theta_i - \sum_{t=1}^T \E{Y_{t,i_t}} 
\]
where $i_t$ is the action selected by the learner in round $t$.

Note that the problem encompasses several common feedback models considered in online learning (modulo the Gaussian assumption), and makes it possible to examine more delicate observation structures:
\begin{description}
\item[Full information:] $\sigma_{ij}=\sigma_j<\infty$ for all $i,j \in \iset{K}$.
\item[Bandit:] $\sigma_{ii}<\infty$ and $\sigma_{ij}=\infty$ for all $i \neq j\in \iset{K}$.
\item[Partial monitoring with feedback graphs \cite{AlCeDe15}:] Each action $i\in \iset{K}$ is associated with an observation set $S_i\subset \iset{K}$ such that $\sigma_{ij}=\sigma_j$ if $j\in S_i$ and $\sigma_{ij}=\infty$ otherwise.
\end{description}
We will call the \emph{uniform variance} version of these problems when all the finite $\sigma_{ij}$ are equal to some $\sigma\ge 0$. 
Some interesting features of the problem can be seen when considering the \emph{generalized full information} case , 
when all entries of $\Sigma$ are finite. In this case, the greedy algorithm, which estimates the payoff of each action by the average of the corresponding observed samples and selects the one with the highest average, achieves at most a constant regret for any time horizon $T$.\footnote{To see this, notice that the error of identifying the optimal action decays exponentially with the number of rounds.}
On the other hand, the constant can be quite large: in particular, when the variance of some observations are large relative to the gaps $\gap_j=\max_i \theta_i - \theta_j$, the situation is rather similar to a partial monitoring setup for a smaller, finite time horizon. In this paper we are going to analyze this problem and present algorithms and lower bounds that are able to ``interpolate'' between these cases and capture the characteristics of the different regimes.

\subsection{Notation}

Let $C_T^{\N} = \{ c\in \N^K \,:\, c_i\ge 0\,, \sum_{i\in\iset{K}} c_i =T \}$ and $N(T) \in C_T$ denote the number of plays over all actions taken by some algorithm in $T$ rounds. Also let $C_T^{\R} = \{ c\in \R^K \,:\, c_i\ge 0\,, \sum_{i\in\iset{K}} c_i =T \}$. We will consider environments with different expected payoff vectors $\theta \in \Theta$, but the variance matrix $\Sigma$ will be fixed. Therefore, an environment can be specified by $\theta$; oftentimes, we will explicitly denote the dependence of different quantities on $\theta$: The probability and expectation functionals under environment $\theta$ will be denoted by $\Prb{\cdot;\theta}$ and $\E{\cdot;\theta}$, respectively. Furthermore, let $i_j(\theta)$ be the $j$th best action (ties are broken arbitrarily, i.e., $\theta_{i_1} \ge \theta_{i_2} \ge \cdots \ge \theta_{i_K}$) and define $\gap_i(\theta)=\theta_{i_1(\theta)}-\theta_i$ for any $i\in \iset{K}$. Then the expected regret under environment $\theta$ is $R_T(\theta) = \sum_{i\in \iset{K}}\E{N_i(T) ; \theta}\gap_i(\theta)$.
For any action $i \in \iset{K}$, let $S_i=\cset{j\in\iset{K}}{\sigma_{ij}<\infty}$ denote the set of actions whose parameter $\theta_j$ is observable by choosing action $i$.
Throughout the paper, $\log$ denotes the natural logarithm and $\simplex^n$ denotes the $n$-dimensional simplex for any positive integer $n$.


\section{Lower Bounds}
\label{sec:lower}

The aim of this section is to derive generic, problem-dependent lower bounds to the regret, which are also able to provide minimax lower bounds.
The hardness in deriving such bounds is that for any fixed $\theta$ and $\Sigma$, the dumb algorithm that always selects $i_1(\theta)$ achieves zero regret (the regret of this algorithm is linear for any $\theta'$ with $i_1(\theta)\neq i_1(\theta')$), so in general it is not possible to give a lower bound for a single instance.  
When deriving asymptotic lower bounds, this is circumvented by only considering \emph{consistent} algorithms whose regret is sub-polynomial for any problem \cite{GrLa97}.
However, this asymptotic notion of consistency is not applicable to finite-horizon problems. Therefore, following \cite{LiMS15}, for any problem we create a family of \emph{related} problems (by perturbing the mean payoffs) such that if the regret of an algorithm is ``too small'' in one of the problems than it will be ``large'' in another one. 

As a warm-up, and to show the reader what form of a lower bound can be expected, first we present an asymptotic lower bound for the  uniform-variance version of the problem of \emph{partial monitoring with feedback graphs}. The result presented below is an easy consequence of \cite{GrLa97}, hence its proof is omitted. An algorithm is said to be \emph{consistent} if
$\sup_{\theta\in \Theta} R_T(\theta) = o(T^\gamma)$ for every $\gamma > 0$. Now assume for simplicity that there is a unique optimal action in environment $\theta$, that is, $\theta_{i_1(\theta)}>\theta_i$ for all $i\neq i_1$ and let
\begin{align*}
C_\theta = \left\{ c\in [0,\infty)^K \,:\, \sum_{i:j\in S_i} c_i \ge \frac{2\sigma^2}{\gap_j^2(\theta)} \, \forall j\ne i_1(\theta) \,, \sum_{i:i_1(\theta)\in S_i} c_i \ge \frac{2\sigma^2}{\gap_{i_2(\theta)}^2(\theta)} \right\} \,.
\end{align*}
Then, for any consistent algorithm and for any $\theta$ with $\theta_{i_1(\theta)}>\theta_{i_2(\theta)}$, 
\begin{align}
\label{eq:lb-asy}
\liminf_{T\rar \infty} \frac{R_T(\theta)}{\log T} \ge \inf_{c\in C_\theta} \inprod{c}{\gap(\theta)} \,.
\end{align}



Note that the right hand side of \eqref{eq:lb-asy} is $0$ for any \emph{generalized full information} problem (recall that the expected regret is bounded by a constant for such problems), but it is a finite positive number for other problems. Similar bounds have been provided in \cite{CaKvLe12,BuErSh14} for graph-structured feedback with self-observability (under non-Gaussian assumptions on the payoffs). In the following we derive finite time lower bounds that are also able to replicate this result.

\subsection{A General Finite Time Lower Bound}

First we derive a general lower bound. For any $\theta, \theta' \in \Theta$ and $q\in \simplex^{|C_T^\N|}$, define $f(\theta,q,\theta')$ as 
\begin{align*}
f(\theta,q,\theta') = & \inf_{q'\in \simplex^{|C_T^{\N}|}} \sum_{a\in C_T^{\N}} q'(a) \inprod{a}{\gap(\theta') } \\
 \text{ s.t. } 
 & \sum_{a\in C_T^{\N}} q(a)\log\frac{q(a)}{q'(a)} \le \sum_{i\in \iset{K}}  \left( I_i(\theta, \theta')\sum_{a\in C_T^{\N}} q(a)a_i \right) \,, 
\end{align*} where $I_i(\theta,\theta')$ is the KL-divergence between $X_{t,i}(\theta)$ and $X_{t,i}(\theta')$, given by $I_i(\theta,\theta') = \KL(X_{t,i}(\theta);X_{t,i}(\theta'))=\sum_{j=1}^K (\theta_j-\theta'_j)^2/2\sigma_{ij}^2$.\todoa{This assumes the coordinates of $X_{t,i}$ are independent!}\todoy{Yes. What if not? is that a interesting setting?}
Clearly, $f(\theta,q,\theta')$ is a lower bound on $R_T(\theta')$ for any algorithm for which the distribution of $N(T)$ is $q$. The intuition behind the allowed values of $q'$ is that we want $q'$ to be as similar to $q$ as the environments $\theta$ and $\theta'$ look like for the algorithm (through the feedback $(X_{t,i_t})_t$). Now define
\begin{align*}
g(\theta,c) = \inf_{q\in \simplex^{|C_T^{\N}|}} \sup_{\theta'\in \Theta}f(\theta,q,\theta'), \qquad \text{ such that } \sum_{a\in C_T^{\N}} q(a)a = c.
\end{align*} 
$g(\theta,c)$ is a lower bound of the worst-case regret of any algorithm 
with $\E{N(T);\theta}=c$. Finally, for any $x > 0$, define
\[
b(\theta,x) = \inf_{c\in C_{\theta,x}} \inprod{c}{\gap(\theta)} \qquad \text{ where }
C_{\theta,x} = \{c\in C_T^\R \,;\, g(\theta,c) \le x \}.
\]
Here $C_{\theta,B}$ contains all the value of $\E{N(T);\theta}$ that can be achieved by some algorithm whose lower bound $g$ on the worst-case regret is smaller than $x$. 
These definitions give rise to the following theorem:
\begin{theorem}
Given any $B>0$, for any algorithm such that $\sup_{\theta' \in \Theta} R_T(\theta) \le B$, we have, for any environment $\theta\in \Theta$, $R_T(\theta) \ge b(\theta,B)$.
\label{thm:lb-0}
\end{theorem}
\begin{remark}
If $B$ is picked as the minimax value of the problem  given the observation structure $\Sigma$, the theorem states that for any minimax optimal algorithm the expected regret for a certain $\theta$ is lower bounded by $b(\theta,B)$.
\end{remark}

\subsection{A Relaxed Lower Bound} 

Now we introduce a relaxed but more interpretable version of the finite-time lower bound of Theorem~\ref{thm:lb-0}, which can be shown to match the asymptotic lower bound \eqref{eq:lb-asy}.
The idea of deriving the lower bound is the following: instead of ensuring that the algorithm performs well in the most adversarial environment $\theta'$, we consider a set of ``bad" environments and make sure that the algorithm performs well on them, where each ``bad'' environment $\theta'$ is the most adversarial one by only perturbing one coordinate $\theta_i$ of $\theta$.

However, in order to get meaningful finite-time lower bounds, we need to perturb $\theta$ more carefully than in the case of asymptotic lower bounds. The reason for this is that for any sub-optimal action $i$, if $\theta_i$ is very close to $\theta_{i_1(\theta)}$, then $\E{N_i(T);\theta}$ is not necessarily small for a good algorithm for $\theta$. If it is small, one can increase $\theta_i$ to obtain an environment $\theta'$ where $i$ is the best action and the algorithm performs bad; otherwise, when $\E{N_i(T);\theta}$ is large, we need to decrease $\theta_i$ to make the algorithm perform badly in $\theta'$. Moreover, when perturbing $\theta_i$ to be better than $\theta_{i_1(\theta)}$, we cannot make $\theta_i'-\theta_{i_1(\theta)}$ arbitrarily small as in asymptotic lower-bound arguments, because when $\theta_i'-\theta_{i_1(\theta)}$ is small, large $\E{N_{i_1(\theta)};\theta'}$ and not necessarily large $\E{N_i(T);\theta'}$ may lead to low finite-time regret in $\theta'$. In the following we make this argument precise to obtain an interpretable lower bound.

\subsubsection{Formulation}

We start with defining a subset of $C_T^\R$ that contains the set of ``reasonable'' values for $\E{N(T);\theta}$.
For any $\theta \in \Theta$ and $B>0$, let 
\[
C'_{\theta,B} = \left\{ c\in C_T^\R \,:\, \sum_{j=1}^K \frac{c_j}{\sigma_{ji}^2} \ge m_i(\theta,B)  \,, \forall i\in\iset{K} \right\}
\]
where $m_i$, the minimum sample size required to distinguish between $\theta_i$ and its worstcase perturbation, is defined as follows: For $i \neq i_1$, if $\theta_{i_1}=D$, then $m_i(\theta,B)=0$. Otherwise
let
\begin{align*}
m_{i,+}(\theta,B)&=\max_{\epsilon \in (\gap_i(\theta), D-\theta_i]}  \frac{1}{\epsilon^2}\log \frac{T(\epsilon-\gap_i(\theta))}{8B}, \\
m_{i,-}(\theta,B)&=\max_{\epsilon \in (0, \theta_i]}  \frac{1}{\epsilon^2}\log \frac{T(\epsilon+\gap_i(\theta))}{8B}, 
\end{align*}
and let $\e_{i,+}$ and $\e_{i,-}$ denote the value of $\e$ achieving the maximum in $m_{i,+}$ and $m_{i,-}$, respectively.
Then, define
\[
m_i(\theta,B)=\begin{cases} m_{i,+}(\theta,B) & \text{ if } \gap_i(\theta)\ge 4B/T; \\
\min\left\{ m_{i,+}(\theta,B), m_{i,-}(\theta,B) \right\} & \text{ if } \gap_i(\theta) < 4B/T~.
\end{cases}
\]
For $i=i_1$, then $m_{i_1}(\theta,B)=0$ if $\theta_{i_2(\theta)}=0$, else the definitions for $i\neq i_1$ change by replacing $\gap_i(\theta)$ with $d_{i_2(\theta)}(\theta)$ (and switching the $+$ and $-$ indices):
let
\begin{align*}
m_{i_1(\theta),-}(\theta,B)
& = \max_{\epsilon \in (\gap_{i_2(\theta)}(\theta),\theta_{i_1(\theta)} ]}  \frac{1}{\epsilon^2}\log \frac{T(\epsilon-\gap_{i_2(\theta)}(\theta))}{8B}, \\
m_{i_1(\theta),+}(\theta,B)
& = \max_{\epsilon \in (0, D-\theta_{i_1(\theta)}]}  \frac{1}{\epsilon^2}\log \frac{T(\epsilon+\gap_{i_2(\theta)}(\theta))}{8B}
\end{align*}
where $\epsilon_{i_1(\theta),-}$ and $\epsilon_{i_1(\theta),+}$ are the maximizers for $\e$ in the above expressions.
Then, define
\[
m_{i_1(\theta)}(\theta,B)=\begin{cases} m_{i_1(\theta),-}(\theta,B) & \text{ if } \gap_{i_2(\theta)}(\theta)\ge 4B/T; \\
\min\left\{ m_{i_1(\theta),+}(\theta,B), m_{i_1(\theta),-}(\theta,B) \right\} & \text{ if } \gap_{i_2(\theta) }(\theta) < 4B/T~.
\end{cases}
\]
Note that $\e_{i,+}$ and $\e_{i,-}$ can be expressed in closed form using the Lambert $W \R\to\R$ function satisfying $W(x)e^{W(x)}=x$:
for any $i\neq i_1(\theta)$,
\begin{align}
\label{eq:eps}
\epsilon_{i,+} & = \min\left\{ D-\theta_i \,,\, \frac{8\sqrt{e}B}{T}e^{W\left(\frac{\gap_i(\theta) T}{16\sqrt{e}B}\right)} + \gap_i(\theta) \right\}~, \\
\epsilon_{i,-} & = 
\min\left\{ \theta_i \,,\, \frac{8\sqrt{e}B}{T}e^{W\left(-\frac{\gap_i(\theta) T}{16\sqrt{e}B}\right)} - \gap_i(\theta) \right\}, \nonumber
\end{align} and similar results hold for $i=i_1$, as well.

Now we can give the main result of this section, a simplified version of Theorem~\ref{thm:lb-0}:
\begin{cor}
\label{cor:lb-0}
Given $B>0$, for any algorithm such that $\sup_{\lambda\in \Theta} R_T(\lambda) \le B$, we have, for any environment $\theta\in \Theta$, $R_T(\theta) \ge b'(\theta,B)=
\min_{c\in C'_{\theta,B}}  \inprod{c}{\gap(\theta)}$.
\end{cor}

Next we compare this bound to existing lower bounds.

\subsubsection{Comparison to the Asymptotic Lower Bound of \eqref{eq:lb-asy}}

Now we will show that our finite time lower bound in Corollary~\ref{cor:lb-0} matches the asymptotic lower bound in \eqref{eq:lb-asy} up to some constants.

Pick $B = \alpha T^\beta$ for some $\alpha>0$ and $0<\beta<1$. 
For simplicity, we only consider $\theta$ which is ``away from" the boundary of $\Theta$ (so that the minimum in \eqref{eq:eps} is not achieved on the boundary) and has a unique optimal action. Then, for $i\ne i_1(\theta)$, it is easy to show that $\epsilon_{i,+} = 
\frac{\gap_i(\theta)}{2W(\gap_i(\theta) T^{1-\beta}/(16\alpha\sqrt{e}))} + \gap_i(\theta)$
by $\eqref{eq:eps}$ and
$m_i(\theta,B) = 
\frac{1}{\epsilon_{i,+}^2}\log \frac{T(\epsilon_{i,+}-\gap_i(\theta))}{8B}$ for
large enough $T$. Then, using the fact that $\log x-\log\log x \le W(x) \le \log x$ for $x\ge e$, it follows that $\lim_{T\to\infty} m_i(\theta,B)/\log T = (1-\beta)/d_i^2(\theta)$, and similarly we can show that $\lim_{T\to\infty} m_{i_1(\theta)}(\theta,B)/\log T = (1-\beta)/d_{i_2(\theta)}^2(\theta)$. Thus, $C'_{\theta,B} \to \frac{(1-\beta) \log T}{2} C_\theta$, under the assumptions of \eqref{eq:lb-asy}, as $T\to\infty$. This implies that Corollary~\ref{cor:lb-0}  matches the asymptotic lower bound of \eqref{eq:lb-asy} up to a factor of $(1-\beta)/2$.


\subsubsection{Comparison to Minimax Bounds}

Now we will show that our $\theta$-dependent finite-time lower bound reproduces the minimax regret bounds of \cite{MaSha11} and \cite{AlCeDe15}, except for the generalized full information case. 

The minimax bounds depend on the following notion of observability:
An action $i$ is \emph{strongly observable} if either $i \in S_i$ or $\iset{K}\setminus\{i\} \subset \{j\,:\,i\in S_j\}$. $i$ is \emph{weakly observable} if it is not strongly observable but there exists $j$ such that $i\in S_j$ (note that we already assumed the latter condition for all $i$). Let $\Wea(\Sigma)$ be the set of all weakly observable actions. $\Sigma$ is said to be strongly observable if $\Wea(\Sigma)=\emptyset$. $\Sigma$ is weakly observable if $\Wea(\Sigma)\ne \emptyset$.

Next we will define two key qualities introduced by \cite{MaSha11} and \cite{AlCeDe15} that characterize the hardness of a problem instance with feedback structure $\Sigma$: 
A set $A\subset \iset{K}$ is called an independent set if for any $i \in A$, $S_i \cap A \subset \{i\}$. The \emph{independence number} $\indp(\Sigma)$ is defined as the cardinality of the largest independent set.
%
For any pair of subsets $A, A' \subset \iset{K}$, $A$ is said to be \emph{dominating} $A'$ if for any $j\in A'$ there exists $i\in A$ such that $j\in S_i$. The \emph{weak domination number} $\dom(\Sigma)$ is defined as the cardinality of the smallest set that dominates $\Wea(\Sigma)$. 

\begin{cor}
\label{cor:lb-minimax}
Assume that $\sigma_{ij}=\infty$ for some $i,j \in \iset{K}$, that is, we are not in the generalized full information case. Then,
\vspace{-0.2cm}
\begin{itemize}
\item[(i)]
if $\Sigma$ is strongly observable, with $B=\alpha\sigma\sqrt{\indp(\Sigma)T}$ for some $\alpha>0$, we have $\sup_{\theta\in \Theta}b'(\theta,B) \ge \frac{\sigma\sqrt{\indp(\Sigma)T}}{64e\alpha}$ for $T\ge 64e^2\alpha^2\sigma^2\indp(\Sigma)^3/D^2$.
\item[(ii)]
If $\Sigma$ is weakly observable, with $B=\alpha(\dom(\Sigma)D)^{1/3}(\sigma T)^{2/3} \log^{-2/3}K$ for some $\alpha>0$, we have $\sup_{\theta\in \Theta}b'(\theta,B) \ge \frac{(\dom(\Sigma)D)^{1/3}(\sigma T)^{2/3} \log^{-2/3}K}{51200e^2\alpha^2} $.
\end{itemize}
\end{cor}

\begin{remark}
In Corollary~\ref{cor:lb-minimax}, picking $\alpha =\frac{1}{8\sqrt{e}}$ for strongly observable $\Sigma$ and $\alpha =\frac{1}{73}$ for weakly observable $\Sigma$ gives formal minimax lower bounds: 
(i) If $\Sigma$ is strongly observable, for any algorithm we have $\sup_{\theta\in \Theta}R_T(\theta) \ge \frac{\sigma\sqrt{\indp(\Sigma)T}}{8\sqrt{e}}$ for $T\ge e\sigma^2\indp(\Sigma)^3/D^2$.
(ii) If $\Sigma$ is weakly observable,  for any algorithm we have $\sup_{\theta\in \Theta}R_T(\theta) \ge \frac{(\dom(\Sigma)D)^{1/3}(\sigma T)^{2/3} }{73\log^{2/3}K } $.

\end{remark}


\section{Algorithms}
\label{sec:algs}

In this section we present two algorithms and their finite-time analysis for the uniform variance version of our problem (where $\sigma_{ij}$ is either $\sigma$ or
$\infty$). The upper bound for the first algorithm matches the asymptotic lower bound in \eqref{eq:lb-asy} up to constants. The second algorithm achieves the minimax lower bounds of Corollary~\ref{cor:lb-minimax} up to logarithmic factors, as well as $O(\log^{3/2}T)$ problem-dependent regret. In the problem-dependent upper bounds of both algorithms, we assume that the optimal action is unique, that is, $\gap_{i_2(\theta)}(\theta)>0$. 

\subsection{An Asymptotically Optimal Algorithm}
\label{sec:alg1}



Let $c(\theta) = \argmin_{c\in C_\theta} \inprod{c}{\gap(\theta)}$; note that increasing $c_{i_1(\theta)}(\theta)$ does not change the value of $\inprod{c}{\gap(\theta)}$ (since $d_{i_1(\theta)}(\theta)=0$), so we take the minimum value of  $c_{i_1(\theta)}(\theta)$ in this definition.
Let $n_i(t) = \sum_{s=1}^{t-1} \I{i\in S_{i_s}}$ be the number of observations for action $i$ before round $t$ and $\htheta_{i,t}$ be the empirical estimate of $\theta_i$ based on the first $n_i(t)$ observations. Let $N_i(t) = \sum_{s=1}^{t-1}\I{i_s=i}$ be the number of plays for action $i$ before round $t$. Note that this definition of $N_i(t)$ is different from that in the previous sections since it excludes the round $t$. 

\begin{algorithm}[ht]
\caption{} 
\label{alg:asym}
\begin{algorithmic}[1]
\STATE Inputs: $\Sigma$, $\beta(n)$, $\alpha$.
\STATE For $t=1,...,K$, observe each action $i$ at least once by playing $i_t$ such that $t\in S_{i_t}$.
\STATE Set exploration count $n_e(K+1)=0$.
\FOR{$t=K+1,K+2,...$}
\IF{$\frac{N(t)}{4\alpha \log t}\in C_{\htheta_t}$}
\STATE Play $i_t = i_1(\htheta_t)$. 
\STATE Set $n_e(t+1)=n_e(t)$.
\ELSE
\IF{$\min_{i\in \iset{K}}n_i(t)<\beta(n_e(t))/K$}
\STATE Play $i_t$ such that $\argmin_{i\in \iset{K}}n_i(t) \in S_{i_t}$.
\ELSE
\STATE Play $i_t$ such that $N_i(t)< c_i(\htheta_t)4\alpha\log t$.
\ENDIF
\STATE Set $n_e(t+1)=n_e(t)+1$.
\ENDIF
\ENDFOR
\end{algorithmic}
\end{algorithm}
Our first algorithm is presented in Algorithm~\ref{alg:asym}.
The main idea, coming from \cite{MaCoPr14}, is that by forcing exploration over all actions the solution $c(\theta)$ of the linear program can be well approximated while paying a constant price. This solves the main difficulty that, without getting enough observations on each action, we may not have good enough estimates for $\gap(\theta)$ and $c(\theta)$. 
One advantage of our algorithm compared to that of \cite{MaCoPr14} is that we use a sublinear exploration schedule $\beta(n)$ instead of a constant rate $\beta(n)=\beta n$. This resolves the problem that, to achieve asymptotically optimal performance, some parameter of the algorithm needs to be chosen according to $\gap_{\min}(\theta)$ as in \cite{MaCoPr14}. The expected regret of Algorithm~\ref{alg:asym} is upper bounded as follows:

\begin{theorem}
For any $\theta\in \Theta$, $\epsilon>0$, $\alpha>2$ and any non-decreasing $\beta(n)$ that satisfies $0\le \beta(n)\le n/2$ and $\beta(m+n)\le \beta(m)+\beta(n)$ for $m,n\in \N$,  
\begin{align*}
 R_T(\theta)
 & \le \left( 2K+2+\frac{4K}{\alpha-2} \right)\gap_{\max}(\theta) + 4K\gap_{\max}(\theta)\sum_{s=0}^T \exp \left( -\frac{\beta(s)\epsilon^2}{2K\sigma^2} \right) \\
& + 2\gap_{\max}(\theta) \beta\left( 4\alpha\log T\sum_{i\in \iset{K}} c_i(\theta,\epsilon)+K \right) + 4\alpha\log T \sum_{i\in \iset{K}} c_i(\theta,\epsilon)\gap_i(\theta) \,. 
\end{align*}
where $c_i(\theta,\epsilon) = \sup\{c_i(\theta')\,:\, |\theta'_j-\theta_j|\le \epsilon \,\, \forall j\in\iset{K}\}$.
\label{thm:alg1-ub1}
\end{theorem}  
Further specifying $\beta(n)$ and using the continuity of $c(\theta)$ around $\theta$, it immediately follows that Algorithm~\ref{alg:asym} achieves asymptotically optimal performance: \begin{cor}
\label{cor:alg1-asym-opt}
 Suppose the conditions of Theorem~\ref{thm:alg1-ub1} hold. Assume, furthermore, that $\beta(n)$ satisfies $\beta(n) = o(n)$ and $\sum_{s=0}^\infty \exp \left( -\frac{\beta(s)\epsilon^2}{2K\sigma^2} \right)<\infty$ for any $\epsilon>0$, then for any $\theta$ such that $c(\theta)$ is unique, 
\[
\limsup_{T\rar\infty} R_T(\theta)/\log T \le 4\alpha \inf_{c\in C(\theta)} \inprod{c}{\gap(\theta)}\,.
\]
\end{cor}
Note that any $\beta(n) = an^b$ with $a\in (0,\frac{1}{2}]$, $b\in (0,1)$ satisfies the requirements in Theorem~\ref{thm:alg1-ub1} and Corollary~\ref{cor:alg1-asym-opt}.
Also note that the algorithms presented in \cite{CaKvLe12,BuErSh14} do not achieve this asymptotic bound.


\subsection{A Minimax Optimal Algorithm}
\label{sec:alg2}

For any $A, A'\subset \iset{K}$, define
$ c(A, A')
 =  \argmax_{c\in \simplex^{|A|}} \min_{i\in A'} \sum_{j:i\in S_j} c_j $
(ties are broken arbitrarily) and $m(A,A') = \min_{i\in A'} \sum_{j:i\in S_j} c_j(A,A')$. For any $A\subset \iset{K}$ and $|A|\ge 2$, define  $A^\Str = \{i\in A \,:\, \exists j\in A , i\in S_j \}$ and  $A^\Wea=A-A^\Str$. 
Furthermore, let $g_{i,r}(\delta)=\sigma\sqrt{\frac{2\log(8K^2r^3/\delta)}{n_i(r)}}$ where $n_i(r)=\sum_{s=1}^{r-1}i_{r,i}$ and $\htheta_{i,r}$ be the empirical estimate of $\theta_i$ based on first $n_i(r)$ observations (i.e., the average of the samples). 

\begin{algorithm}[ht]
\caption{} 
\label{alg:minimax}
\begin{algorithmic}[1]
\STATE Inputs: $\Sigma$, $\delta$.
\STATE Set $t_1=0$, $A_1=\iset{K}$.
\FOR{$r=1,2,...$}
\STATE Let $\alpha_r = \min_{1\le s\le r, A_s^\Wea\ne \emptyset} m(\iset{K}, A_s^\Wea )$ and $\gamma(r)=(\sigma\alpha_r t_r/D)^{2/3}$. ( Define $\alpha_r=1$ if $A_s^\Wea=\emptyset$ for all $1\le s\le r$.) 
\IF{$A_r^\Wea\ne \emptyset$ and $\min_{i\in A_r^\Wea} n_i(r)< \min_{i\in A_r^\Str} n_i(r)$ and $\min_{i\in A_r^\Wea} n_i(r)< \gamma(r)$}
\STATE Set $c_r = c(\iset{K},A_r^\Wea )$.
\ELSE
\STATE Set $c_r = c(A_r, A_r^\Str )$.
\ENDIF
\STATE Play $i_r = \lc c_r\cdot\norm{c_r}_0 \rc$.
\STATE $t_{r+1}\lar t_r+\norm{i_r}_1$.
\STATE $A_{r+1}\lar \{i\in A_r \,:\, \htheta_{i,r+1}+g_{i,r+1}(\delta)\ge \max_{j\in A_r} \htheta_{j,r+1} -g_{j,r+1}(\delta)\}$.
\IF{$|A_{r+1}|=1$}
\STATE Play the only action in the remaining rounds.
\ENDIF
\ENDFOR
\end{algorithmic}
\end{algorithm}

Our second algorithm, presented in Algorithm~\ref{alg:minimax}, follows a successive elimination process: it explores all possibly optimal actions (called ``good actions" later) based on some confidence intervals until only one action remains. While doing exploration, it first tries to explore the good actions by only using good ones. However, due to weak observability, some good actions might have to be explored by the actions that are eliminated. To control this exploration-exploitation trade off, we use a sublinear function $\gamma$ to control the  exploration of weakly observable actions.  In the following we present high-probability bounds on the performance of the algorithm, so, with a slight abuse of notation, $R_T(\theta)$ will denote the regret without expectation in the rest of this section. 

\begin{theorem}
\label{thm:ub-alg2}
For any $\delta\in(0,1)$ and any $\theta\in \Theta$,
\begin{align*}
R_T(\theta) \le (\dom(\Sigma)D)^{1/3}(\sigma T)^{2/3} \cdot 7\sqrt{6\log(2KT/\delta)}+125\sigma^2 K^3/D+ 13K^3D 
\end{align*}
with probability at least $1-\delta$  if $\Sigma$ is weakly observable, while
\begin{align*}
R_T(\theta) \le2KD + 80\sigma \sqrt{\indp(\Sigma) T\cdot 6\log K \log\frac{2KT}{\delta}}
\end{align*}
with probability at least $1-\delta$  if $\Sigma$ is strongly observable.
\end{theorem}

\begin{theorem}[Problem-dependent upper bound]
For any $\delta\in (0,1)$ and any $\theta\in \Theta$ such that the optimal action is unique, with probability at least $1-\delta$,
\begin{align*}
R_T(\theta) 
& \le \frac{1603\dom(\Sigma)D\sigma^2}{\gap_{\min}^2(\theta)} \left( \log\frac{2KT}{\delta} \right)^{3/2} 
+ 14K^3D + \frac{125\sigma^2K^3}{D}\\
& +  15\left( \dom(\Sigma)D\sigma^2\right)^{1/3}\left(\frac{125\sigma^2}{D^2}+10 \right)K^2 \left( \log\frac{2KT}{\delta} \right)^{1/2}
 \,.
\end{align*}
\label{thm:ub2-alg2}
\end{theorem}

\begin{remark}
Picking $\delta=1/T$ gives an $O\left(\log^{3/2}T \right)$ upper bound on the expected regret.
\end{remark}
\begin{remark}
Note that Algortihm~\ref{alg:minimax} is similar to the  UCB-LP algorithm of \cite{BuErSh14}, which admits a better problem-dependent upper bound (although does not achieve it with optimal problem-dependent constants), but it does not achieve the minimax bound even under strong observability.
\end{remark}


\section{Conclusions and Open Problems}
\label{sec:conc}
We considered a novel partial-monitoring setup with Gaussian side observations, which generalizes 
the recently introduced setting of graph-structured feedback, allowing finer quantification of the observed information  from one action to another.
We provided non-asymptotic problem-dependent lower bounds that imply existing asymptotic problem-dependent and non-asymptotic minimax lower bounds (up to some constant factors) beyond the full information case. We also provided an algorithm that achieves the asymptotic problem-dependent lower bound (up to some universal constants) and another algorithm that achieves the minimax bounds under both weak and strong observability.

However, we think this is just the beginning. For example, we currently have no algorithm that achieves both the problem dependent and the minimax lower bounds at the same time. 
Also, our upper bounds only correspond to the graph-structured feedback case. It is of great interest to 
go beyond the weak/strong observability in characterizing the harness of the problem, and provide algorithms that can adapt to any correspondence between the mean payoffs an the variances (the hardness is that one needs to identify suboptimal actions with good information/cost trade-off).

\subsubsection*{Acknowledgments}

This work was supported by the Alberta Innovates Technology Futures 
through the Alberta Ingenuity Centre for Machine Learning (AICML)
and NSERC. During this work, A. Gy\"orgy was with the Department of 
Computing Science, University of Alberta.

\bibliographystyle{unsrt}
\bibliography{main}

\newpage
\appendix
\section{Proofs for Section~\ref{sec:lower}}

\subsection{
Proof of Theorem~\ref{thm:lb-0}
}

Let $\phi_{\theta, \sigma}$ denote the density function of a $K$-dimensional Gaussian random variable with mean vector $\theta$ and independent components wehere the variance of the $i$th coordinate is $\sigma_i^2$, and define
$L_T=\sum_{t=1}^T \log \frac{\phi_{\theta, \sigma_{i_t}}(X_{t,i_t})}{\phi_{\theta',\sigma_{i_t}}(X_{t,i_t})}$ where $i_t$ is the choice of the algorithm in round $t$.
Let $q, q'\in \simplex^{|C_T^{\N}|}$ denote the joint distribution over the number of plays for each action under environment $\theta$ and $\theta' \in \Theta$, respectively, that is, $q(a) = \Prb{N(T)=a; \theta}$ and $q'(a) = \Prb{N(T)=a; \theta'}$ for each $a\in C_T^\N$.


For any $a \in C_T^\N$, applying a standard change of measure equality (see, e.g., \cite[Lemma~15]{KaCaGa14}), 
we obtain
\begin{align*}
q'(a) & = \Prb{N(T)=a;\theta'} = \E{\I{N(T)=a}\exp(-L_T); \theta}\\
&  = \E{\I{N(T)=a}\E{\exp(-L_T)|N(T)=a ; \theta}; \theta}\\
& \ge \E{\I{N(T)=a}\exp\left(\E{-L_T|N(T)=a ; \theta}\right); \theta}\\
& = \Prb{N(T)=a;\theta} \exp\left(\E{-L_T|N(T)=a ; \theta}\right)\\
& =  q(a)\exp\left(\E{-L_T|N(T)=a ; \theta}\right) \,.
\end{align*}
Thus $\E{L_T|N(T)=a ; \theta} \ge \log\frac{q(a)}{q'(a)}$ and so
\begin{align*}
\sum_{i\in \iset{K}} \E{N_i(T);\theta} I_i(\theta,\theta')&  = \E{L_T;\theta} \\
&= \sum_{a\in C_T^\N} \Prb{N(T)=a;\theta}\E{L_T|N(T)=a ; \theta} \ge   \sum_{a\in C_T^\N} q(a)\log\frac{q(a)}{q'(a)} \,,
\end{align*}
where $\E{N_i(T);\theta} =\sum_{a\in C_T^{\N}} q(a)a_i $.
Therefore, according to the definition of $f(\theta,q,\theta')$, we have $f(\theta,q,\theta') \le \sum_{a\in C_T^{\N}} q'(a) \inprod{a}{\gap(\theta') } = R_T(\theta')$ for any $\theta' \in \Theta$.  Then $\sup_{\theta'\in \Theta}f(\theta,q,\theta') \le \sup_{\theta'\in \Theta} R_T(\theta') \le B$ must hold. Since $\E{N(T);\theta} = \sum_{a\in C_T^{\N}} q(a)a$ we have $g(\theta,\E{N(T);\theta}) \le \sup_{\theta'\in \Theta}f(\theta,q,\theta') \le B$. Thus $\E{N(T);\theta}\in C_{\theta,B}$ and so $R_T(\theta) \ge b(\theta,B)$, which concludes the proof of Theorem~\ref{thm:lb-0}.


\subsection{Proof of Corollary~\ref{cor:lb-0}}
\label{app:rlower}

We start the proof with two technical lemmas on the Lambert $W$ function.

\begin{lemma}
\label{lemma:opteps1}
Let $a, b >0$ with $ab<1$ and $f(x)=\frac{1}{x^2}\log \left((x+a)b\right)$ for $x>0$. Then $f(x)\le f(x_*)$ for all $x>0$ where 
\begin{align*}
x_* = \frac{\sqrt{e}}{b}e^{W\left(-\frac{ab}{2\sqrt{e}}\right)} - a~.
\end{align*}
\end{lemma}

\begin{lemma}
\label{lemma:opteps2}
Let $a, b >0$ and $f(x)=\frac{1}{x^2}\log \left((x-a)b\right)$ for $x>a$. Then $f(x)\le f(x_*)$ for all $x>a$ where 
\begin{align*}
x_* = \frac{\sqrt{e}}{b}e^{W\left(\frac{ab}{2\sqrt{e}}\right)} + a~.
\end{align*}
\end{lemma}

\begin{proof}[Proof of Lemma~\ref{lemma:opteps2}]

\begin{align*}
f'(x) = \frac{x^{-3}}{x-a}\left( x-2(x-a)\log\left( (x-a)b \right)\right) \,.
\end{align*}

Let $g(y)=y+a-2y\log by$ defined on $y>0$.
\begin{align*}
g'(y) = - 2\log yb -1
\end{align*}
So $g(y)$ is increasing when $0<y<\frac{1}{b\sqrt{e}}$ and decreasing when $y>\frac{1}{b\sqrt{e}}$.

Since $\lim_{y\rar 0} g(y)=a>0$ and $\lim_{y\rar +\infty}g(y)=-\infty$ we know that there exists a unique $y_*>0$ such that $g(y_*)=0$, $g(y)>0$ for $0<y<y_*$ and $g(y)<0$ for $y>y_*$. It can be verified that $y_*=x_*-a=\frac{\sqrt{e}}{b}e^{W\left(\frac{ab}{2\sqrt{e}}\right)}$ satisfies $g(y_*)=0$. Therefore $f'(x)>0$ when $a<x<x_*$ and $f'(x)<0$ when $x>x_*$. Since $f(x)$ is continuous when $x>a$ we have proved that $f(x)\le f(x_*)$ for all $x>a$.

\end{proof}

\begin{proof}[Proof of Corollary~\ref{cor:lb-0}]

To prove the corollary, it suffices to show $b'(\theta,B)\le b(\theta,B)$. 

Define $C'_{\theta,B} = \left\{ c\in C_T^\R \,:\, \sum_{j=1}^K \frac{c_j}{\sigma_{ji}^2} \ge m_i(\theta,B)  \,, \forall i\in\iset{K} \right\}$. We will prove $C_{\theta,B} \subset C'_{\theta,B}$ by showing that if $c\in C_T^\R$ satisfies $g(\theta,c)\le B$ then $c\in C'_{\theta,B}$. 

For $c\in C_T^\R$, if $g(\theta,c)\le B$, then there exists $q\in \simplex^{|C_T^{\N}|}$ such that $\sup_{\theta'\in \Theta}f(\theta,q,\theta') \le B$ and $\sum_{a\in C_T^{\N}}q(a)a=c$. We will next derive $K$ constraints on $c$ to show that $c\in C'_{\theta,B}$ by picking different $\theta'$s.  Before proceeding with the proof, we introduce the following technical lemma:

\begin{lemma}
\label{lemma:cod1}
For any $A\subset C_T^\N$ and $q,q'\in \simplex^{|C_T^{\N}|} $, 
if $q(A) \ge 1/2$ then
\begin{align*}
\sum_{a\in C_T^{\N}} q(a)\log\frac{q(a)}{q'(a)} \ge \frac{1}{2} \log \frac{1}{4q'(A)},
\end{align*}
where $q'(A)=\sum_{a\in A}q'(a)$.
\end{lemma}
\begin{proof}
Let $A^c = C_T^\N - A$. By the log-sum inequality we have
\begin{equation}
\label{eq:logsum}
\sum_{a\in C_T^{\N}} q(a)\log\frac{q(a)}{q'(a)} \ge \KL(q(A), q'(A)) \,, 
\end{equation}
where for $x,y \in [0,1]$, $\KL(x,y)=x\log(x/y)+(1-x)\log((1-x)/(1-y))$ denotes the binary KL-divergence. Now for such $x,y$, since $x \log x + (1-x) \log (1-x)$ is minimized for $x=1/2$, we have
\[
\KL(x,y) \ge \log\frac{1}{2} + x\log\frac{1}{y} + (1-x)\log(\frac{1}{1-y})
\ge \log\frac{1}{2} + \frac{1}{2} \log\frac{1}{y}= \frac{1}{2}\log\frac{1}{4y}\,.
\]
Combining with \eqref{eq:logsum} proves the lemma.
\end{proof}

Now we continue the proof of Corollary~\ref{cor:lb-0}. First consider $i\ne i_1(\theta)$.

If $\sum_{a:a_i\le T/2} q(a) \ge 1/2$, construct $\theta^{(i,+)}$ by replacing $\theta_i$ with $\theta_i+\epsilon_{i,+}$. Then $f(\theta,q,\theta^{(i,+)}) \le B$ holds, so there exists $q'\in \simplex^{|C_T^{\N}|}$ such that $\sum_{a\in C_T^{\N}} q'(a) \inprod{a}{\gap(\theta^{(i,+)}) } \le B$ and  $\sum_{a\in C_T^{\N}} q(a)\log\frac{q(a)}{q'(a)} \le \sum_{j\in \iset{K}}  c_j I_j(\theta, \theta^{(i,+)})$. Applying Lemma~\ref{lemma:cod1} with $A=\{a:a_i\le T/2\}$ gives
\begin{align*}
\sum_{j\in \iset{K}}  c_j I_j(\theta, \theta^{(i,+)}) \ge \frac{1}{2} \log\frac{1}{4q'(A)} \,,
\end{align*}
where
\begin{align*}
q'(A)
& = \sum_{a\in C_T^\N} \I{\sum_{j\ne i}a_j\ge T/2} q'(a)  \le \frac{2}{T} \sum_{a\in C_T^\N} q'(a)\sum_{j\ne i}a_j \\
& =   \frac{2}{T(\epsilon_{i,+}-\gap_i(\theta))} \sum_{a\in C_T^\N} q'(a)\sum_{j\ne i}a_j(\epsilon_{i,+}-\gap_i(\theta))\\
& \le \frac{2}{T(\epsilon_{i,+}-\gap_i(\theta))} \sum_{a\in C_T^{\N}} q'(a) \inprod{a}{\gap(\theta^{(i,+)}) } \\
& \le \frac{2B}{T(\epsilon_{i,+}-\gap_i(\theta))} \,.
\end{align*}

Since $I_j(\theta, \theta^{(i,+)})=\epsilon_{i,+}^2/2\sigma_{ji}^2$, we get
\begin{align*}
\sum_{j\in \iset{K}} \frac{c_j}{\sigma_{ji}^2} \ge  \frac{1}{\epsilon_{i,+}^2}\log \frac{T(\epsilon_{i,+}-\gap_i(\theta))}{8B} \,.
\addtocounter{equation}{1}\tag{\theequation} \label{eq:lb-cor-1}
\end{align*}

If $\sum_{a:a_i\le T/2} q(a) < 1/2$ and $\gap_i(\theta)\ge 4B/T$, then 
\begin{align*}
f(\theta,q,\theta)
&  =  \sum_{a\in C_T^{\N}} q(a) \inprod{a}{\gap(\theta) }
\ge \sum_{a\in C_T^{\N}} q(a) a_i\gap_i(\theta) \\
& \ge  \gap_i(\theta)\sum_{a\in C_T^{\N}} \I{a_i\ge T/2}q(a) a_i \\ 
& \ge \frac{4B}{T} \frac{T}{2}\sum_{a\in C_T^{\N}} \I{a_i\ge T/2}q(a) > B \,,  
\end{align*}
which contradicts the fact that $\sup_{\theta'\in \Theta}f(\theta,q,\theta') \le B$.

If $\sum_{a:a_i\le T/2} q(a) < 1/2$ and $\gap_i(\theta)< 4B/T$, construct $\theta^{(i,-)}$ by replacing $\theta_i$ with $\theta_i-\epsilon_{i,-}$. Then  there exists $q'\in \simplex^{|C_T^{\N}|}$ such that $\sum_{a\in C_T^{\N}} q'(a) \inprod{a}{\gap(\theta^{(i,-)}) } \le B$ and  $\sum_{a\in C_T^{\N}} q(a)\log\frac{q(a)}{q'(a)} \le \sum_{j\in \iset{K}}  c_j I_j(\theta, \theta^{(i,-)})$. Applying Lemma~\ref{lemma:cod1} with $A=\{a:a_i> T/2\}$ gives
\begin{align*}
\sum_{j\in \iset{K}}  c_j I_j(\theta, \theta^{(i,-)}) \ge \frac{1}{2} \log\frac{1}{4q'(A)} \,,
\end{align*}
where
\begin{align*}
q'(A) & = \sum_{a\in C_T^\N} \I{a_i> T/2} q'(a) 
 \le \frac{2}{T} \sum_{a\in C_T^\N} a_i q'(a)
 \le \frac{2}{T(\epsilon_{i,-}+\gap_i(\theta))} \sum_{a\in C_T^\N} q'(a)a_i(\epsilon_{i,-}+\gap_i(\theta))\\
& \le \frac{2}{T(\epsilon_{i,-}+\gap_i(\theta))} \sum_{a\in C_T^\N} q'(a) \inprod{a}{\gap(\theta^{(i,-)}) }
 \le \frac{2B}{T(\epsilon_{i,-}+\gap_i(\theta))} \,.
\end{align*}

Using $I_j(\theta, \theta^{(i,-)})=\epsilon_{i,-}^2/2\sigma_{ji}^2$ gives
\begin{align*}
\sum_{j\in \iset{K}} \frac{c_j}{\sigma_{ji}^2} \ge  \frac{1}{\epsilon_{i,-}^2}\log \frac{T(\epsilon_{i,-}+\gap_i(\theta))}{8B} \,.
\addtocounter{equation}{1}\tag{\theequation} \label{eq:lb-cor-2}
\end{align*}

Now consider $i=i_1(\theta)$.

If $\sum_{a:a_i \ge T/2} q(a) \ge 1/2$, construct $\theta^{(i_1,-)}$ by replacing $\theta_i$ with $\theta_i-\epsilon_{i,-}$. Then there exists $q'\in \simplex^{|C_T^{\N}|}$ such that $\sum_{a\in C_T^{\N}} q'(a) \inprod{a}{\gap(\theta^{(i,-)}) } \le B$ and  $\sum_{a\in C_T^{\N}} q(a)\log\frac{q(a)}{q'(a)} \le \sum_{j\in \iset{K}}  c_j I_j(\theta, \theta^{(i,-)})$. Applying Lemma~\ref{lemma:cod1} with $A=\{a:a_i\ge T/2\}$ and
\begin{align*}
q'(A) & = \sum_{a\in C_T^\N} \I{a_i\ge  T/2} q'(a) 
 \le \frac{2}{T(\epsilon_{i,-}-\gap_{i_2(\theta)}(\theta))} \sum_{a\in C_T^\N}q'(a)a_i(\epsilon_{i,-}-\gap_{i_2(\theta)}(\theta))
 \le \frac{2B}{T(\epsilon_{i,-}-\gap_{i_2(\theta)}(\theta))} 
\end{align*}
gives
\begin{align*}
\sum_{j\in \iset{K}} \frac{c_j}{\sigma_{ji}^2} \ge  \frac{1}{\epsilon_{i,-}^2}\log \frac{T(\epsilon_{i,-}-\gap_{i_2(\theta)}(\theta))}{8B} \,.
\addtocounter{equation}{1}\tag{\theequation} \label{eq:lb-cor-3}
\end{align*}

If $\sum_{a:a_i \ge T/2} q(a) < 1/2$ and $\gap_{i_2(\theta)}(\theta) \ge 4B/T$, then 
\begin{align*}
f(\theta,q,\theta)
&  =  \sum_{a\in C_T^{\N}} q(a) \inprod{a}{\gap(\theta) }
\ge \sum_{a\in C_T^{\N}} q(a) \gap_{i_2(\theta)} \sum_{j\ne i}a_j  \ge  \gap_{i_2(\theta)} \sum_{a\in C_T^{\N}} \I{ \sum_{j\ne i}a_j> T/2}q(a)  \sum_{j\ne i}a_j \\ 
& > \frac{4B}{T} \frac{T}{2}\sum_{a\in C_T^{\N}} \I{ \sum_{j\ne i}a_j > T/2}q(a) \ge B \,,  
\end{align*}
which contradicts the fact that $\sup_{\theta'\in \Theta}f(\theta,q,\theta') \le B$.

If $\sum_{a:a_i \ge T/2} q(a) < 1/2$ and $\gap_{i_2(\theta)}(\theta)< 4B/T$, construct $\theta^{(i,+)}$ by replacing $\theta_{i}$ with $\theta_{i}+\epsilon_{i,+}$. Then there exists $q'\in \simplex^{|C_T^{\N}|}$ such that $\sum_{a\in C_T^{\N}} q'(a) \inprod{a}{\gap(\theta^{(i,+)}) } \le B$ and  $\sum_{a\in C_T^{\N}} q(a)\log\frac{q(a)}{q'(a)} \le \sum_{j\in \iset{K}}  c_j I_j(\theta, \theta^{(i,+)})$. Applying Lemma~\ref{lemma:cod1} with $A=\{a:a_i< T/2\}$ and
\begin{align*}
q'(A)
& = \sum_{a\in C_T^\N} \I{\sum_{j\ne i}a_j> T/2} q'(a)  \le \frac{2}{T} \sum_{a\in C_T^\N} q'(a)\sum_{j\ne i}a_j \\
& =   \frac{2}{T(\epsilon_{i,+}+\gap_{i_2(\theta)}(\theta))} \sum_{a\in C_T^\N} q'(a)\sum_{j\ne i}a_j(\epsilon_{i,+}+\gap_{i_2(\theta)}) \le \frac{2B}{T(\epsilon_{i,+}+\gap_{i_2(\theta)})} 
\end{align*}
gives
\begin{align*}
\sum_{j\in \iset{K}} \frac{c_j}{\sigma_{ji}^2} \ge  \frac{1}{\epsilon_{i,+}^2}\log \frac{T(\epsilon_{i,+}+\gap_{i_2(\theta)})}{8B} \,.
\addtocounter{equation}{1}\tag{\theequation} \label{eq:lb-cor-4}
\end{align*}

Combining \eqref{eq:lb-cor-1} \eqref{eq:lb-cor-2} \eqref{eq:lb-cor-3} \eqref{eq:lb-cor-4} gives $c\in C'_{\theta,B}$, which concludes the proof.

\end{proof}

\subsection{Proof of Corollary~\ref{cor:lb-minimax}}
\begin{proof}[Proof of Corollary~\ref{cor:lb-minimax}]

Define $\epsilon = \frac{8eB}{T}$. First consider the case that $\Sigma$ is strongly observable.

If the maximum independence number $\indp(\Sigma)\ge 2$, there exists an independent set $A_\indp \subset \iset{K}$ such that $|A_\indp|=\indp(\Sigma)$. We construct $\theta$ as follows: Let $\theta_{i_1}=D/2$ for some $i_1\in A_\indp$ and $\theta_i=D/2-\epsilon$ for $i\in A_\indp \setminus \{i_1\}$. For the remaining $i\notin A_\indp$, let $\theta_i=0$. Note that each $i$ in $A_\indp$ must be self observable since otherwise it is a weakly observable action. Also in $A_\indp$  $i$ can be observed only by itself according to the definition of independent sets.  

Then we will lower bound $b'(\theta,B)$. According to our choice of $\epsilon$, we have
\begin{align*}
\frac{8\sqrt{e}B}{T}e^{W\left(\frac{\epsilon T}{16\sqrt{e}B} \right)} + \epsilon = 2\epsilon \,.
\end{align*}
Therefore, for $i=i_1$ we have $\epsilon_{i,-}=2\epsilon$ and $\epsilon_{i,+}=2\epsilon$ for $i\in A_\indp \setminus \{i_1\}$. Thus for any $i\in A_\indp$,
\begin{align*}
m_i(\theta,B) = \frac{1}{4\epsilon^2}\log\frac{T\epsilon}{8B} = \frac{1}{4\epsilon^2} \,.
\end{align*}

Recall that we defined $C'_{\theta,B} = \left\{ c\in C_T^\R \,:\, \sum_{j:i\in S_j} c_j \ge \sigma^2m_i(\theta,B)  \,, \forall i\in\iset{K} \right\}$ and $b'(\theta,B)=\inf_{c\in C'_{\theta,B}} \inprod{c}{\gap(\theta)}$. For any $c\in C'_{\theta,B}$, let $a=\sum_{i\notin A_\indp}c_i$. Then we have for any $i\in A_\indp$, $\sum_{j:i\in S_j} c_j \le a + c_i$ and thus $c_i \ge \sigma^2m_i(\theta,B) -a=\frac{\sigma^2}{4\epsilon^2}-a$. Since $\gap_i(\theta)=\epsilon$ for all $i\in A_\indp \setminus \{i_1\}$ and $\gap_i(\theta)=D/2$ for all $i\notin A_\indp$, we get
\begin{align*}
\inprod{c}{\gap(\theta)}
& = \sum_{i\in A_\indp\setminus \{i_1\}} c_i \epsilon + \frac{aD}{2} \ge (\indp(\Sigma)-1)\left(\frac{\sigma^2}{4\epsilon^2}-a \right)\epsilon + \frac{aD}{2} \\
& \ge \frac{\indp(\Sigma)}{2}\left(\frac{\sigma^2}{4\epsilon^2}-a \right)\epsilon + \frac{aD}{2} = \frac{\indp(\Sigma)\sigma^2}{8\epsilon}+ \frac{D-\indp(\Sigma)\epsilon}{2}a\\
& \ge \frac{\indp(\Sigma)\sigma^2}{8\epsilon}
\addeq\label{eq:lb-str-1}
\end{align*}
if $\indp(\Sigma)\epsilon<D$. Applying our particular choice of $\epsilon$ and $B$, we get the conclusion that for $T\ge \frac{64e^2\alpha^2\sigma^2\indp(\Sigma)^3}{D^2}$, $b'(\theta,B) \ge \frac{\sigma\sqrt{\indp(\Sigma)T}}{64e\alpha}$.

If $\indp(\Sigma)=1$, since we exclude the full information case, there always exists a pair of actions $i_1$ and $i_2$ such that $i_2\notin S_{i_1}$ (here $i_1\ne i_2$ is not necessary). We construct $\theta$ by setting $\theta_{i_1}=D/2$ and $\theta_{i}=D/2-\epsilon$ for all $i\ne i_1$. Then $    m_i(\theta,B)= \frac{1}{4\epsilon^2}$ holds for all $i\in \iset{K}$. For any $c\in C'_{\theta,B}$, let $a=\sum_{i\ne i_1}c_i$, then $\sum_{j:i_2\in S_j} c_j \le a$. Hence $a\ge \sigma^2m_{i_2}(\theta,B)=\frac{\sigma^2}{4\epsilon^2}$ and 
\begin{align*}
\inprod{c}{\gap(\theta)} = a\epsilon \ge \frac{\sigma^2}{4\epsilon} > \frac{\indp(\Sigma)\sigma^2}{8\epsilon}\,.
\addeq\label{eq:lb-str-2}
\end{align*}
Combining \eqref{eq:lb-str-1} and \eqref{eq:lb-str-2} gives the first part of Corollary~\ref{cor:lb-minimax}.

Now we turn to the case that $\Sigma$ is weakly observable. The idea of constructing the worst $\theta$ comes from the proof of Theorem~7 in \citep{AlCeDe15} which based on the following graph-theoretic lemma:

\begin{lemma}[Restated from Lemma~8 in \citep{AlCeDe15}]
Let $G=(V,E)$ be a directed graph with $K$ vertices and let $W\subset V$ be a subset of vertices with domination number $\dom$. Then there exists an independent set $U\subset W$ with the property that $|U|\ge \frac{\dom}{50\log K}$ and any vertex of $G$ dominates at most $\log K$ vertices of $U$.    
\label{lemma:dom-indp}
\end{lemma} 

Let $\Wea(\Sigma) \subset \iset{K}$ be the set of all weakly observable actions. By Lemma~\ref{lemma:dom-indp} we know that there exists an independent set $A_\dom \subset \Wea(\Sigma)$ such that $|A_\dom|\ge \frac{\dom(\Sigma)}{50\log K}$ and for any $i\in \iset{K}$, $|S_i \cap U|\le \log K$.

If $\dom(\Sigma) \ge 100\log K$ such that $|A_\dom|\ge 2$, we can construct $\theta$ as follows: Let $\theta_{i_1}=D/2$ for some $i_1\in A_\dom$ and $\theta_i=D/2-\epsilon$ for $i\in A_\dom \setminus \{i_1\}$. For the remaining $i\notin A_\dom$, let $\theta_i=0$. Note that actions in $A_\dom$ cannot be observed by any action inside $A_\dom$. For any $c\in C'_{\theta,B}$, let $a=\sum_{i\notin A_\dom} c_i$. Since for any $i$, $|S_i \cap U|\le \log K$, we have $\sum_{i\in A_\dom} \sum_{j:i\in S_j} c_j \le a \log K$ and
\begin{align*}
a\log K \ge |A_\dom| \min_{i\in A_\dom}\sum_{j:i\in S_j} c_j \ge|A_\dom|\min_{i\in A_\dom} \sigma^2 m_i(\theta,B) \ge \frac{\dom(\Sigma)\sigma^2}{200\log K\epsilon^2} \,. 
\end{align*}
Therefore,
\begin{align*}
\inprod{c}{\gap(\theta)} \ge \frac{aD}{2} \ge \frac{\dom(\Sigma)\sigma^2D}{200\epsilon^2 \log^2 K} 
= \frac{(\dom(\Sigma)D)^{1/3}(\sigma T)^{2/3} \log^{-2/3}K}{12800e^2\alpha^2} \,.
\addeq\label{eq:lb-wea-1}
\end{align*}

If $\dom(\Sigma) < 100\log K$, then we pick a weakly observable action as $i_2$. There must be another action $i_1$ such that $i_2\notin S_{i_1}$ due to the definition of weakly observable actions. Then we set $\theta$ as $\theta_{i_1}=D/2$, $\theta_{i_2}=D/2-\epsilon$ and $\theta_{i}=0$ for the remaining actions. So for any $c\in C'_{\theta,B}$, let $a=\sum_{i\ne i_1, i_2} c_i \ge \sigma^2m_{i_2}(\theta,B)$. Then
\begin{align*}
\inprod{c}{\gap(\theta)} 
& \ge \frac{aD}{2} 
\ge  \frac{\sigma^2m_{i_2}(\theta,B)D}{2} = \frac{D\sigma^2}{8\epsilon^2} 
= \frac{D^{1/3}(\sigma T)^{2/3}}{512e^2 \alpha^2} \cdot \frac{\log ^{4/3}K}{\dom(\Sigma)^{2/3}}\\
& \ge \frac{(\dom(\Sigma)D)^{1/3}(\sigma T)^{2/3}\log^{-2/3}K}{51200 e^2 \alpha^2} \,.
\addeq\label{eq:lb-wea-2}
\end{align*}
In the last step we used the fact that $K\ge 3$ for any weakly observable $\Sigma$. 

Combining \eqref{eq:lb-wea-1} and \eqref{eq:lb-wea-2} gives the second part of Corollary~\ref{cor:lb-minimax}.

\end{proof}



\section{Proofs for Section~\ref{sec:alg1}}

\subsection{Proof of Theorem~\ref{thm:alg1-ub1}}
\begin{proof}[Proof of Theorem~\ref{thm:alg1-ub1}]

Define events 
\[
U_t = \left\{ \forall i\in \iset{K}, |\htheta_{i,t} -\theta_i|\le \sqrt{\frac{2\alpha\sigma^2\log t}{n_i(t)}} \right\} \,,
\]
\[
V_t = \left\{\forall i\in \iset{K}, |\htheta_{i,t} -\theta_i|\le \epsilon \right\} \,,
\]
\[
W_t = \left\{ \frac{N(t)}{4\alpha \log t}\in C(\htheta_t) \right\} \,,
\]
\[
Y_t = \left\{ \min_{i\in \iset{K}}n_i(t)<\beta(n_e(t))/K \right\}
\]
and $U_t^c$, $V_t^c$, $W_t^c$, $Y_t^c$ be their complements.

\begin{align*}
R_T(\theta)
& = \sum_{t=1}^T \E{\gap_{i_t}(\theta)} \le K\gap_{\max}(\theta) + \sum_{t=K+1}^n \E{\gap_{i_t}(\theta)} \\
& = K\gap_{\max}(\theta) + \sum_{t=K+1}^T \E{\gap_{i_t}(\theta)\left( \I{U_t^c} + \I{U_t,W_t} + \I{U_t,W_t^c,Y_t} \right.\right. \\ 
& \left.\left. + \I{U_t,W_t^c,Y_t^c,V_t^c} + \I{U_t,W_t^c,Y_t^c,V_t} \right)} \,.
\addeq \label{eq:regret-decomp}
\end{align*}

Then we will upper bound each quantity in \eqref{eq:regret-decomp} separately.

By Hoeffding's inequality, we have
\begin{align*}
\Prb{|\htheta_{i,t}-\theta_i|> \sqrt{\frac{2\alpha\sigma^2\log t}{n_i(t)}} } \le 2t^{1-\alpha} \,,
\end{align*}
where we use a union bound over all possible $n_i(t)$.

Then $\sum_{t=K+1}^n\E{\gap_{i_t}(\theta)\I{U_t^c}}$ can be bounded by 
\begin{align*}
\sum_{t=K+1}^T\E{\gap_{i_t}(\theta) \I{U_t^c}} 
\le \gap_{\max}(\theta)\sum_{t=K+1}^T \Pr(U_t^c) 
\le \gap_{\max}(\theta)\sum_{t=K+1}^T 2Kt^{1-\alpha} 
\le \frac{2K\gap_{\max}(\theta)}{\alpha-2} \,.
\addeq \label{eq:regret-term1}
\end{align*}

Next consider $\sum_{t=K+1}^T\E{\gap_{i_t}(\theta)\I{U_t,W_t}}$. If $U_t$ and $W_t$ hold, first we have
\begin{align*}
n_{i_1(\htheta_t)} \ge \frac{8\alpha\sigma^2\log t}{\gap^2_{i_1(\htheta_t)}(\htheta_t)} \,,
\end{align*}
and
\begin{align*}
\htheta_{i_1(\htheta_t),t} - \theta_{i_1(\htheta_t)} 
\le \sqrt{\frac{2\alpha\sigma^2\log t}{n_{i_1(\htheta_t)}(t)}}  
\le \frac{\gap_{i_1(\htheta_t)}(\htheta_t)}{2}
\le \frac{\gap_i(\htheta_t)}{2}
\addeq\label{eq:emp-best-1}
\end{align*}
for any $i\ne i_1(\htheta_t)$. Similarly, for $i\ne i_1(\htheta_t)$ we have 
\begin{align*}
\theta_i - \htheta_{i,t} \le \sqrt{\frac{2\alpha\sigma^2\log t}{n_i(t)}} \le  \frac{\gap_i(\htheta_t)}{2} \,. \addeq\label{eq:emp-best-i}
\end{align*}

Combining \eqref{eq:emp-best-1} and \eqref{eq:emp-best-i} gives $\theta_i\le\theta_{i_1(\htheta_t)} $ for any $i\ne i_1(\htheta_t)$, which means $i_1(\htheta_t) = i_1(\theta)$, hence
\begin{align*}
\sum_{t=K+1}^T\E{\gap_{i_t}(\theta)\I{U_t,W_t}} = 0 \,.
\addeq\label{eq:regret-term2}
\end{align*}

Consider the next term in \eqref{eq:regret-decomp},
\begin{align*}
 \sum_{t=K+1}^T \E{\gap_{i_t}(\theta)   \I{U_t,W_t^c,Y_t }}  \le \gap_{\max}(\theta)\E{\sum_{t=K+1}^T \I{U_t,W_t^c,Y_t} } \,.
\addeq \label{eq:forced-expl-0}
\end{align*}

To upper bound \eqref{eq:forced-expl-0}, we will first prove:
\begin{proposition}
\begin{align*}
\sum_{t=K+1}^T \I{W_t^c,Y_t} \le 1+ \beta\left(\sum_{t=K+1}^T \I{W_t^c}\right) \,.
\addeq \label{eq:forced-expl}
\end{align*}
\end{proposition}

\begin{proof}[Proof of \eqref{eq:forced-expl}]
According to the algorithm we have $n_e(t) = \sum_{s=K+1}^{t-1}\I{W_s^c}$ for $t>K$, we then proceed by the following proposition:

\begin{proposition}
For $K<t_1<t_2$, if $\sum_{s=t_1}^{t_2-1} \I{W_s^c,Y_s} \ge K$, then $\min_{i\in\iset{K}}n_i(t_2)\ge \min_{i\in\iset{K}}n_i(t_1)+1$.
\label{prop:expl-count}
\end{proposition}
\begin{proof}[Proof of Proposition~\ref{prop:expl-count}]
If for such $t_1$ and $t_2$, $\min_{i\in\iset{K}}n_i(t_2) = \min_{i\in\iset{K}}n_i(t_1)$, 
then there must exist $j$ such that $n_j(t_1) = n_j(t_2)$ and $n_j(s) = \min_{i\in\iset{K}}n_i(s)$ for all $t_1\le s \le t_2$. Since  $\sum_{s=t_1}^{t_2-1} \I{W_s^c,Y_s} \ge K$, there exist $K$ instants $t_1\le s_1<s_2<...<s_K\le t_2-1$ such that $\left\{W_{s_k}^c,Y_{s_k}\right\}$ happens for $1\le k \le K$. According to the algorithm, for each $s_k$, there exists $j'\ne j$ such that $j'\in S_{i_{s_k}}$ and $n_{j'}(s_k)=n_{j}(s_k)=\min_{i\in\iset{K}}n_i(s_k)$. Note that each action appears at most once as such $j'$ for $1\le k\le K$ since $n_{j'}(s_k+1)=n_{j'}(s_k)+1$, but there are only $K-1$ actions other than $j$, which means such $j$ cannot exist. Hence $\min_{i\in\iset{K}}n_i(t_2)\ge \min_{i\in\iset{K}}n_i(t_1)+1$ is proved.
\end{proof}

Now we define 
\begin{align*}
t' = \max\left\{ K+1\le t\le T \,:\, W_t^c, Y_t \right\} \,.
\end{align*}
If such $t'$ does not exist, then \eqref{eq:forced-expl} must hold. If such $t'$ exists, by Proposition~\ref{prop:expl-count},
\begin{align*}
\min_{i\in \iset{K}}n_i(t')
 \ge \min_{i\in \iset{K}} n_i(K+1) + \left\lfloor \frac{1}{K}\sum_{t=K+1}^{t'-1} \I{W_t^c,Y_t } \right\rfloor  \ge \frac{1}{K}\sum_{t=K+1}^{t'-1} \I{W_t^c,Y_t } \,.
\end{align*}

Therefore,
\begin{align*}
\sum_{t=K+1}^{T} \I{W_t^c,Y_t }
& = 1+ \sum_{t=K+1}^{t'-1} \I{W_t^c,Y_t }  \le 1+ K \min_{i\in \iset{K}}n_i(t')  < 1+ \beta(n_e(t')) \\
& \le 1+ \beta(n_e(T)) \le 1+ \beta\left( \sum_{t=K+1}^T \I{W_t^c} \right) 
\end{align*}
gives \eqref{eq:forced-expl}.
\end{proof}

Now continue with \eqref{eq:forced-expl-0}
\begin{align*}
& \sum_{t=K+1}^T \I{U_t,W_t^c,Y_t}
\le \sum_{t=K+1}^{T} \I{W_t^c,Y_t }
\le 1+ \beta\left( \sum_{t=K+1}^T \I{W_t^c} \right) \\
& \le 1 + \beta\left(\sum_{t=K+1}^T \I{U_t^c} + \I{U_t, W_t^c, Y_t} + \I{U_t, W_t^c, Y_t^c, V_t^c}
+ \I{U_t, W_t^c, Y_t^c, V_t}
 \right) \\
& \le 1 + \frac{1}{2}\sum_{t=K+1}^T \left( \I{U_t^c}  + \I{U_t, W_t^c, Y_t} +  \I{U_t, W_t^c, Y_t^c, V_t^c} \right)
+ \beta\left( \I{U_t, W_t^c, Y_t^c, V_t} \right) \,.
\end{align*}
Thus we have
\begin{align*}
& \sum_{t=K+1}^T \I{U_t,W_t^c,Y_t}\\ 
& \le 
2 + \sum_{t=K+1}^T \I{U_t^c} + \sum_{t=K+1}^T  \I{U_t, W_t^c, Y_t^c, V_t^c} + 2\beta\left( \sum_{t=K+1}^n\I{U_t, W_t^c, Y_t^c, V_t} \right) \,,
\end{align*}
and 
\begin{align*}
& \sum_{t=K+1}^T \E{\gap_{i_t}(\theta)   \I{U_t,W_t^c,Y_t }}  \le \gap_{\max}(\theta)\E{\sum_{t=K+1}^T \I{U_t,W_t^c,Y_t} } \\
& \le 2\gap_{\max}(\theta) + \frac{2K\gap_{\max}(\theta)}{\alpha-2} + \gap_{\max}(\theta)\sum_{t=K+1}^T \E{ \I{U_t, W_t^c, Y_t^c, V_t^c} } \\
& + 2\gap_{\max}(\theta)\E{\beta\left( \sum_{t=K+1}^n\I{U_t, W_t^c, Y_t^c, V_t} \right)}
\addeq \label{eq:regret-term3}
\end{align*}
by applying \eqref{eq:regret-term1}.

To bound $\sum_{t=K+1}^T\E{\I{U_t,W_t^c,Y_t^c,V_t^c}}$, we first introduce two lemmas from \cite{CoPr14} (Lemma~2.1 and 2.2):

\begin{lemma}
\label{lemma:beta-expl-0}
Let $\{Z_t\}_{t\in\N^+}$ be a sequence of independent random variables from $\Normal(0,\sigma^2)$. Define $\F_t$ the $\sigma$-algebra generated by $\{Z_s\}_{s\le t}$ and the filtration $\F=(\F_t)_{t\in\N^+}$. Consider $r,n_0 \in \N^+$ and $T\ge n_0$. Define $Y_t=\sum_{s=n_0}^{t-1} B_sZ_s$ where $B_t\in \{0,1\}$ is an $\F_{t-1}$-measurable random variable. Further define $n(t) = \sum_{s=n_0}^{t-1} B_s$ and $\phi$ an $\F$-stopping time which satisfies either $n(\phi)\ge r$ or $\phi=T+1$.
 
Then we have
\begin{align*}
\Pr\left( |Y_{\phi}|> n(\phi)\epsilon, \phi\le T\right) \le 2\exp \left( -\frac{r\epsilon^2}{2\sigma^2} \right) \,.
\end{align*}

\end{lemma}

\begin{lemma}
\label{lemma:beta-expl}
Define $\F_t$ the $\sigma$-algebra generated by $\{X_{i,s}\}_{s\in \iset{t},i\in \iset{K}}$. Let $\Lambda \subset [1,T]\cap \N$ be a set of (random) time instants. Assume there exists a sequence of (random) sets $\{\Lambda_s\}_{0\le s \le T}$ such that (i) $\Lambda \subset \cup_{0\le s \le T}\Lambda_s$, (ii) for all $0\le s \le T$, $|\Lambda_s|\le 1$, (iii) for all $0\le s \le T$, if $t\in \Lambda_s$ then $n_i(t)\ge \beta(s)/K$, and (iv) the event $\{t\in \Lambda_s\}$ is $\F_{t}$ measurable. Then for any $\epsilon>0$ and $i\in \iset{K}$:
\begin{align*}
\E{ \sum_{t=1}^T \I{t\in \Lambda, |\htheta_{i,t}-\theta_i|>\epsilon}} \le  \sum_{s=0}^T 2\exp \left( -\frac{\beta(s)\epsilon^2}{2K\sigma^2} \right) \,.
\end{align*}
\end{lemma}

\begin{proof}[Proof of Lemma~\ref{lemma:beta-expl}]
We adapt the proof of Lemma~2.2 from \cite{CoPr14}. For $0\le s\le T$, define $\phi_s=t$ if $\Lambda_s=\{t\}$ or $\phi_s=T+1$ if $\Lambda_s=\emptyset$. Then 
\begin{align*}
& \E{ \sum_{t=1}^T \I{t\in \Lambda, |\htheta_{i,t}-\theta_i|>\epsilon}} \le \E{ \sum_{s=0}^T \I{\phi_s\le T, |\htheta_{i,\phi_s}-\theta_i|>\epsilon}} \\
& = \sum_{s=0}^T \Prb{\phi_s\le T, |\htheta_{i,\phi_s}-\theta_i|>\epsilon}.
\addeq \label{eq:lemma-beta-expl}
\end{align*}
Since $\phi_s$ can be viewed as an $\F$-stopping time and satisfies either $n_i(\phi_s)\ge \lc \beta(s)/K\rc$ or $\phi_s=T+1$, if $\lc \beta(s)/K\rc \ge 1$ then applying Lemma~\ref{lemma:beta-expl-0} gives 
\begin{align*}
\Prb{\phi_s\le T, |\htheta_{i,\phi_s}-\theta_i|>\epsilon} 
\le 2\exp \left( -\frac{\lc \beta(s)/K\rc \epsilon^2}{2\sigma^2} \right) \le 2\exp \left( -\frac{\beta(s)\epsilon^2}{2K\sigma^2} \right) \,. 
\end{align*}
If $\lc \beta(s)/K\rc = 0$ then $\Prb{\phi_s\le T, |\htheta_{i,\phi_s}-\theta_i|>\epsilon} < 2 = 2\exp \left( -\frac{\beta(s)\epsilon^2}{2K\sigma^2} \right)$ still holds. Now proceeding from \eqref{eq:lemma-beta-expl} we can get the result of Lemma~\ref{lemma:beta-expl}. 

\end{proof}

Now we define $\Lambda = \{t: K+1\le t\le T, U_t, W_t^c,Y_t^c \}$, and $\Lambda_s = \{t: K+1\le t\le T, U_t, W_t^c, n_e(t) =s, \min_{i\in\iset{K}}n_i(t)\ge \beta(s)/K \}$. It can be verified that $\Lambda_s$ satisfies the conditions in Lemma~\ref{lemma:beta-expl}: (i) If $t\in \Lambda$ then there must be some $0\le s\le T$ such that $n_e(t)=s$ and thus $t\in \Lambda_s$. (ii) If $t\in \Lambda_s$ then for $t'>t$, $n_e(t')\ge n_e(t+1) = n_e(t)+1=s+1$, so $t'\notin \Lambda_s$. Condition (iii) and (iv) are also satisfied from the definition of $\Lambda_s$.

Then
\begin{align*}
& \sum_{t=K+1}^T\E{\I{U_t,W_t^c,Y_t^c,V_t^c}} = \sum_{t=K+1}^T\E{\I{t\in \Lambda,V_t^c}}\\
& \le \sum_{i=1}^K \sum_{t=K+1}^T\E{\I{t\in \Lambda,|\htheta_{i,t}-\theta_i|>\epsilon}}  \le  2K \sum_{s=0}^T \exp \left( -\frac{\beta(s)\epsilon^2}{2K\sigma^2} \right) \,.
\addeq \label{eq:regret-term4}
\end{align*}

Finally we will upper bound $\sum_{t=K+1}^n\gap_{i_t}(\theta)\I{U_t, W_t^c, Y_t^c, V_t}$. 

Recall that in the algorithm, if $W_t^c$ and $Y_t^c$ happens,  some $i_t$ satisfying $N_i(t)< c_i(\htheta_t)4\alpha\log t$ is played. Such $i_t$ must exist because otherwise $\frac{N_i(t)}{4\alpha \log t} \ge c_i(\htheta_t)4\alpha\log t$ holds for any $i\in \iset{K}$ and thus $W_t = \left\{\frac{N(t)}{4\alpha \log t}\in C(\htheta_t) \right\}$ happens, which causes contradiction.

Define
\begin{align*}
\Theta(\theta,\epsilon) = \left\{ \lambda\in \Theta \,:\, \forall i\in \iset{K}, |\lambda_i-\theta_i|\le \epsilon \right\} \,,
\end{align*}
and
\begin{align*}
c_i(\theta,\epsilon) = \sup_{\lambda\in \Theta(\theta,\epsilon)} c_i(\lambda)\,.
\end{align*}

Let $T_i$ be the maximum $t\le T$ such that $i_t=i$ and $\I{U_t, W_t^c, Y_t^c, V_t}=1$. Then 
\begin{align*}
N_i(T_i) = \sum_{s=1}^{T_i-1}\I{i_s=i} \le c_i(\htheta_{T_i})4\alpha\log T_i \le c_i(\theta,\epsilon)4\alpha\log T\,.
\end{align*}
Thus
\begin{align*}
\sum_{t=K+1}^{T} \I{i_t=i, U_t, W_t^c, Y_t^c, V_t} \le c_i(\theta,\epsilon)4\alpha\log T +1 \,.
\end{align*}
So we have
\begin{align*}
\sum_{t=K+1}^T\gap_{i_t}(\theta)\I{U_t, W_t^c, Y_t^c, V_t} \le 4\alpha\log T \sum_{i\in \iset{K}} c_i(\theta,\epsilon)\gap_i(\theta) + \sum_{i\in \iset{K}}\gap_i(\theta) \,,
\addeq\label{eq:regret-term5}
\end{align*}
and
\begin{align*}
\sum_{t=K+1}^T \I{U_t, W_t^c, Y_t^c, V_t} \le 4\alpha\log T\sum_{i\in \iset{K}} c_i(\theta,\epsilon)+K \,.
\addeq\label{eq:expl-cost}
\end{align*}

Now plugging \eqref{eq:expl-cost} \eqref{eq:regret-term4} into \eqref{eq:regret-term3} and plugging \eqref{eq:regret-term1} \eqref{eq:regret-term2} \eqref{eq:regret-term3}  \eqref{eq:regret-term4} \eqref{eq:regret-term5} into \eqref{eq:regret-decomp} we get
\begin{align*}
R_T(\theta)
& \le \left( 2K+2+\frac{4K}{\alpha-2} \right)\gap_{\max}(\theta) + 4K\gap_{\max}(\theta)\sum_{s=0}^T \exp \left( -\frac{\beta(s)\epsilon^2}{2K\sigma^2} \right) \\
& + 2\gap_{\max}(\theta) \beta\left( 4\alpha\log T\sum_{i\in \iset{K}} c_i(\theta,\epsilon)+K \right) + 4\alpha\log T \sum_{i\in \iset{K}} c_i(\theta,\epsilon)\gap_i(\theta) \,.
\end{align*}

\end{proof}

\section{Proofs for Section~\ref{sec:alg2}}

\subsection{Proof of Theorem~\ref{thm:ub-alg2}}
\begin{proof}[Proof of Theorem~\ref{thm:ub-alg2}]

For every $r>0$, define the events
\[
U_r = \left\{ \forall i\in \iset{K}, |\htheta_{i,r}-\theta_i|\le g_{i,r}(\delta)\right\}\,.
\]
Then, by Hoeffding's inequality and union bound, we have 
\begin{align*}
\Pr(\forall r\ge 2, U_r) \ge 1-\delta~.
\end{align*}

Next we will upper bound the regret based on the fact that  $U_r$ holds for all $r\ge 2$.
Define $r_T=\max\{r \,:\, t_t<T, |A_r|\ge 2\}$, the event 
\[
V_r=\left\{A_r^\Wea\ne \emptyset, \min_{i\in A_r^\Wea} n_i(r)< \min\{ \min_{i\in A_r^\Str} n_i(r), \gamma(r)\}\right\} 
\] and its complement $V_r^c$. Then consider the regret:
\begin{align*}
R_T(\theta)
& \le \sum_{r=1}^{r_T} \I{V_r}\inprod{i_r}{\gap(\theta)} + \sum_{r=1}^{r_T} \I{V_r^c}\inprod{i_r}{\gap(\theta)} \\
& \le \sum_{r=1}^{r_T} \I{V_r} \norm{i_r}_1D + \sum_{r=1}^{r_T} \I{V_r^c} \norm{i_r}_1 \max_{i\in A_r} \gap_i(\theta)\,.
\addeq \label{eq:alg2-regret}
\end{align*}

We upper bound the two terms in \eqref{eq:alg2-regret} separately. Before proceeding, we introduce the following proposition which lower bounds $n_i(r)$ for $i\in A_r^\Wea$.

\begin{proposition}
For any $i, r$ such that $i\in A_r^\Wea$, 
\begin{align*}
n_i(r) \ge \frac{\alpha_{r-1}}{2} \sum_{s=1}^{r-1}\I{V_s}\norm{i_s}_1 -(\beta_r-1)K \,,
\addeq \label{eq:weak-expl}
\end{align*}
where $\beta_r = \left\vert \bigcup_{1\le s\le r} A_s^\Wea \right \vert$.
\label{prop:weak-expl}
\end{proposition}

\begin{proof}[Proof of Proposition~\ref{prop:weak-expl}]
The proof is done by induction. Let $W_r$ denote the event that for any $1\le s\le r$ and any $i\in A_s^\Wea$, \eqref{eq:weak-expl} holds. $W_1$ holds because $A_1^\Wea = \emptyset$. Now we assume $W_r$ holds and consider $W_{r+1}$.

If $A_{r+1}^\Wea = \emptyset$, then $W_{r+1}$ holds. If $A_{r+1}^\Wea \ne \emptyset$, for $i\in A_{r+1}^\Wea$, consider $n_i(r+1)$ in different cases:

If $i\in A_r^\Wea$, then $n_i(r) \ge \frac{\alpha_{r-1}}{2} \sum_{s=1}^{r-1}\I{V_s}\norm{i_s}_1 -(\beta_r-1)K$. Recall that $\alpha_r = \min_{1\le s\le r, A_s^\Wea\ne \emptyset} m(\iset{K}, A_s^\Wea )$. So we have 
\begin{align*}
n_i(r+1) 
\ge n_i(r) + \I{V_r} \norm{c_r}_0\alpha_r \ge \frac{\alpha_{r}}{2} \sum_{s=1}^{r}\I{V_s}\norm{i_s}_1 -(\beta_{r+1}-1)K \,,
\end{align*}
where we use the fact that $\alpha_r$ is non-increasing, $\beta_r$ is non-decreasing as well as 
\begin{align*}
\norm{i_r}_1 
= \norm{\lc c_r\cdot\norm{c_r}_0 \rc}_1 
\le \norm{c_r}_0 + \norm{c_r}_0 \cdot \norm{c_r}_1 
= 2\norm{c_r}_0 \,.
\addeq \label{eq:rounding-approx}
\end{align*}

If $i\notin A_r^\Wea$, then $i \in A_s^\Str$ for all $1\le s \le r$ and thus $\beta_{r+1}\ge \beta_r+1$. Let $r' = \max \{s\le r \,:\, V_s \}$. If such $r'$ does not exist, then 
\begin{align*}
n_i(r+1) 
\ge 0 \ge \frac{\alpha_{r}}{2} \sum_{s=1}^{r}\I{V_s}\norm{i_s}_1 -(\beta_{r+1}-1)K \,.
\end{align*}
If such $r'$ exists
\begin{align*}
n_i(r+1) 
& \ge n_i(r') 
> \min_{j\in A_{r'}^\Wea} n_j(r')  \ge \frac{\alpha_{r'-1}}{2} \sum_{s=1}^{r'-1}\I{V_s}\norm{i_s}_1 -(\beta_{r'}-1)K \\
& \ge \frac{\alpha_{r}}{2} \sum_{s=1}^{r}\I{V_s}\norm{i_s}_1 - \frac{\alpha_{r}}{2}\norm{i_{r'}}_1 -(\beta_{r'}-1)K  \ge   \frac{\alpha_{r}}{2} \sum_{s=1}^{r}\I{V_s}\norm{i_s}_1 -\beta_{r'}K \\
& \ge \frac{\alpha_{r}}{2} \sum_{s=1}^{r}\I{V_s}\norm{i_s}_1 -(\beta_{r+1}-1)K \,,
\end{align*}
where the facts $\alpha_r \le 1$, $\norm{i_{r'}}_1\le 2K$ and $\beta_{r'} \le \beta_{r+1}-1$ are used.

Now we have proved that $W_{r+1}$ holds based on the assumption of $W_r$, hence $W_r$ holds for any $r$, which gives the result of Proposition~\ref{prop:weak-expl}.

\end{proof}

Based on Proposition~\ref{prop:weak-expl}, $\sum_{s=1}^{r}\I{V_s}\norm{i_s}_1$ can be upper bounded by the following fact:

\begin{proposition}
For any $r\ge 1$, $\sum_{s=1}^{r}\I{V_s}\norm{i_s}_1 \le \frac{2\gamma(r)+2K\beta_r}{\alpha_r}$.
\label{prop:weak-expl-2}
\end{proposition}

\begin{proof}[Proof of Proposition~\ref{prop:weak-expl-2}]

Let $r' = \max \{s\le r \,:\, V_s \}$. Then
\begin{align*}
\gamma(r') > \min_{i\in A_{r'}^\Wea} n_i(r')  \ge \frac{\alpha_{r'-1}}{2} \sum_{s=1}^{r'-1}\I{V_s}\norm{i_s}_1 -(\beta_{r'}-1)K \,.
\end{align*}
Hence
\begin{align*}
\sum_{s=1}^{r}\I{V_s}\norm{i_s}_1 
& \le \sum_{s=1}^{r'-1}\I{V_s}\norm{i_s}_1 + \norm{i_{r'}}_1 
\le \frac{2\gamma(r')+2K(\beta_{r'}-1)}{\alpha_{r'}} + 2K \\
& \le \frac{2\gamma(r')+2K\beta_{r'}}{\alpha_{r'}}\,.
\end{align*}
Since $\alpha_r$ is non-increasing, $\beta_r$ is non-decreasing and $\gamma(r)/\alpha_r = \alpha_r^{-1/3}(\sigma t_r/D)^{2/3}$ is non-decreasing, we have  $\sum_{s=1}^{r}\I{V_s}\norm{i_s}_1 \le \frac{2\gamma(r)+2K\beta_r}{\alpha_r}$.

\end{proof}

Now we are ready to upper bound the first term in \eqref{eq:alg2-regret}:
\begin{align*}
\sum_{r=1}^{r_T} \I{V_r} \norm{i_r}_1D 
\le \frac{2\gamma(r_T)+2K\beta_{r_T}}{\alpha_{r_T}}D
= 2\alpha_{r_T}^{-1/3}D^{1/3}(\sigma T)^{2/3} + 2KD\frac{\beta_{r_T}}{\alpha_{r_T}} \,.
\addeq\label{eq:alg2-regret-term1}
\end{align*}

Next consider the second term in \eqref{eq:alg2-regret}: $\sum_{r=1}^{r_T} \I{V_r^c} \norm{i_r}_1 \max_{i\in A_r} \gap_i(\theta)$. Given $U_r$ holds for all $r$ we know that $i_1(\theta)$ is never eliminated. Then for any $i\in A_r$, we have $|\htheta_{i,r}-\theta_i|\le g_{i,r}(\delta)$ and $\htheta_{i,r}+g_{i,r}(\delta)\ge \htheta_{i_1(\theta)} - g_{i_1(\theta),r}(\delta)$. Therefore, 
\begin{align*}
\gap_i(\theta) \le \min \left\{ D, 2g_{i,r}(\delta)+2g_{i_1(\theta),r}(\delta) \right\} \le \min \left\{ D, 4\sigma\sqrt{6\log\frac{2KT}{\delta}}\left( \min_{i\in A_r}n_i(r) \right)^{-1/2} \right\} \,.
\end{align*}
So
\begin{align*}
\sum_{r=1}^{r_T} \I{V_r^c} \norm{i_r}_1 \max_{i\in A_r} \gap_i(\theta) \le \sum_{r=1}^{r_T} \I{V_r^c} \norm{i_r}_1 \min \left\{ D,C ( \min_{i\in A_r}n_i(r) )^{-1/2} \right\} \,,
\addeq\label{eq:alg2-term2-0}
\end{align*}
where $C=4\sigma\sqrt{6\log\frac{2KT}{\delta}}$. 

The next step is to lower bound $\min_{i\in A_r}n_i(r)$ when $V_r^c$ happens. Define $\eta_{\min}=\min_{A\in \iset{K},|A|\ge 2}m(A,A^\Str)$. For $i\in A_r^\Str$,
\begin{align*}
n_i(r) \ge \sum_{s=1}^{r-1} \I{V_s^c}\norm{c_s}_0 m(A_s,A_s^\Str) 
\ge \frac{\eta_{\min}}{2} \sum_{s=1}^{r-1} \I{V_s^c}\norm{i_s}_1 \,.
\addeq\label{eq:alg2-str-obsv}
\end{align*}

For $i \in A_r^\Wea$, since $V_r^c$ happens and $A_r^\Wea \ne \emptyset$, we have 
\begin{align*}
n_i(r) \ge \min\{ \min_{i\in A_r^\Str} n_i(r), \gamma(r)\} \ge \min\left\{ \frac{\eta_{\min}}{2} \sum_{s=1}^{r-1} \I{V_s^c}\norm{i_s}_1, \gamma(r) \right\} \,.
\end{align*}
By Proposition~\ref{prop:weak-expl-2}, 
\begin{align*}
 \frac{\eta_{\min}}{2} \sum_{s=1}^{r-1} \I{V_s^c}\norm{i_s}_1 
& \ge \frac{1}{2K} \left( t_r-\sum_{s=1}^{r} \I{V_s}\norm{i_s}_1 \right) \ge \frac{1}{2K} \left( t_r-\frac{2\gamma(r)+2K\beta_r}{\alpha_r} \right) \\
& =  \frac{1}{2K} \left( t_r - 2 \alpha_r^{-1/3}\left(\frac{\sigma t_r}{D}\right)^{2/3}-2K\beta_r/\alpha_r  \right)\\
& \ge \frac{1}{2K} t_r - \left(\frac{\sigma t_r}{D}\right)^{2/3} -K^2 \,,
\end{align*}
where we used $\alpha_r, \eta_{\min} \ge 1/K$ and $\beta_r\le K$.

For $t_r\ge \frac{125\sigma^2}{D^2}K^3+10K^3$, we have $\frac{4}{5}t_r \ge 4K\left(\frac{\sigma t_r}{D}\right)^{2/3}$ and $\frac{1}{5}t_r \ge 2K^3$, so
\begin{align*}
 \frac{\eta_{\min}}{2} \sum_{s=1}^{r-1} \I{V_s^c}\norm{i_s}_1 
& \ge \frac{1}{2K} t_r - \left(\frac{\sigma t_r}{D}\right)^{2/3} -K^2 \\
& \ge 2\left(\frac{\sigma t_r}{D}\right)^{2/3} + K^2 - \left(\frac{\sigma t_r}{D}\right)^{2/3} -K^2 \\
& = \left(\frac{\sigma t_r}{D}\right)^{2/3} \ge \left(\frac{\sigma \alpha_r t_r}{D}\right)^{2/3} = \gamma(r) \,.
\end{align*}

So we have proved that for any $r\le r_T$ such that $t_r\ge T_0 = \frac{125\sigma^2}{D^2}K^3+10K^3$ and $V_r^c$ happens, $\min_{i\in A_r}n_i(r)\ge \gamma(r)\ge (\sigma\alpha_{r_T} t_r/D)^{2/3}$. Therefore, following \eqref{eq:alg2-term2-0} gives
\begin{align*}
& \sum_{r=1}^{r_T} \I{V_r^c} \norm{i_r}_1 \max_{i\in A_r} \gap_i(\theta) \\
& \le \sum_{r=1}^{r_T} \I{V_r^c} \norm{i_r}_1 \min \left\{ D,C ( \min_{i\in A_r}n_i(r) )^{-1/2} \right\} \\
& \le \sum_{r\ge 1: t_r<T_0} \norm{i_r}_1 D + \sum_{r\le r_T: t_r\ge T_0} \norm{i_r}_1 C\left(\frac{\sigma \alpha_{r_T}}{D}\right)^{-1/3} t_r^{-1/3} \\
& \le (T_0+2K)D +  C\left(\frac{\sigma \alpha_{r_T}}{D}\right)^{-1/3}  \sum_{r\le r_T: t_r\ge T_0} (t_{r+1}-t_r)(t_{r+1}-2K)^{-1/3} \\
& \le (T_0+2K)D +  C\left(\frac{\sigma \alpha_{r_T}}{D}\right)^{-1/3} \int_{T_0}^{t_{r_T+1}} (x-2K)^{-1/3} dx \\
& \le (T_0+2K)D +  C\left(\frac{\sigma \alpha_{r_T}}{D}\right)^{-1/3} \int_{T_0-2K}^{t_{r_T}} x^{-1/3} dx \\
& \le (T_0+2K)D +  \frac{3}{2}C\left(\frac{\sigma \alpha_{r_T}}{D}\right)^{-1/3} T^{2/3} \\
& =  \frac{125\sigma^2 K^3}{D}+ (10K^3+2K)D +  \alpha_{r_T}^{-1/3}D^{1/3}(\sigma T)^{2/3} \cdot 6\sqrt{6\log\frac{2KT}{\delta}} \,. 
\addeq \label{eq:alg2-regret-term2}
\end{align*}
Now plugging \eqref{eq:alg2-regret-term1} and \eqref{eq:alg2-regret-term2} into \eqref{eq:alg2-regret} gives
\begin{align*}
R_T(\theta) \le \alpha_{r_T}^{-1/3}D^{1/3}(\sigma T)^{2/3} \cdot 7\sqrt{6\log\frac{2KT}{\delta}}+\frac{125\sigma^2 K^3}{D}+ 13K^3D \,. 
\end{align*}

If $\Sigma$ is strongly observable, then $A_r^\Wea$ is always empty and $V_r^c$ always happens. According to \eqref{eq:alg2-regret} \eqref{eq:alg2-term2-0} and \eqref{eq:alg2-str-obsv} we have 
\begin{align*}
R_T(\theta)
& \le  \sum_{r=1}^{r_T} \norm{i_r}_1 \max_{i\in A_r} \gap_i(\theta) \\
& \le \sum_{r=1}^{r_T}  (t_{r+1}-t_r) \min \left\{ D, C\left(\frac{\eta_{\min}}{2}\right)^{-1/2} t_r^{-1/2} \right\} \\
& \le 2KD + C\left(\frac{\eta_{\min}}{2}\right)^{-1/2} \int_{0}^{t_{r_T}} x^{-1/2} dx \\
& \le 2KD + 8\sigma \sqrt{\frac{T}{\eta_{\min}} \cdot 12\log\frac{2KT}{\delta}} \,.
\end{align*}

To finish the proof, it suffices to show that $\frac{1}{\alpha_{r_T}}\le \dom(\Sigma)$ and 
$\frac{1}{\eta_{\min}}\le \indp(\Sigma)50 \log K $, which is based on the following fact:

\begin{proposition}
For any $A,A'\subset \iset{K}$ Let $\domlp(A,A')$ denote the minimum fractional cover number from $A$ to $A'$, that is 
\begin{align*}
\domlp(A,A') = & \min_{b\in [0,\infty)^A} \sum_{i\in A} b_i \\
\text{ s.t. } & \sum_{i:j\in S_i} b_i \ge 1 \,\, \forall j\in A' \,.
\end{align*}
Then $m(A, A') = \frac{1}{\domlp(A,A')}$.
\label{prop:domlp}
\end{proposition}

\begin{proof}[Proof of Proposition~\ref{prop:domlp}]

Recall that 
\begin{align*}
m(A, A')
& =  \max_{c\in \simplex^A} \min_{i\in A'} \sum_{j:i\in S_j} c_j \\
& =  \max_{c\in \simplex^A , a} a 
\text{ s.t. } \sum_{i:j\in S_i} c_i \ge a \,\, \forall j\in A' \,.
\end{align*}
Let $b = c/a$, then
\begin{align*}
m(A, A')
& =  \max_{b \in [0,\infty)^A , a} a \text{ s.t. }  \sum_{i:j\in S_i} b_i \ge 1 \,\, \forall j\in A' \text{ and } \sum_{i\in A} b_i = \frac{1}{a} \\
& = \max_{b \in [0,\infty)^A} \frac{1}{\sum_{i\in A} b_i} \text{ s.t. } \sum_{i:j\in S_i} b_i \ge 1 \,\, \forall j\in A' \\
& = \frac{1}{\domlp(A,A')} \,.
\end{align*}

\end{proof}

To lower bound $\alpha_{r_T}$, let $\dom(A,A')$ be the integer version of $\domlp(A,A')$ by restricting $b\in \N^A$. Then we have $\dom(\Sigma) = \dom(\iset{K}, \Wea(\Sigma))$ and 
\begin{align*}
\alpha_{r_T} \ge m(\iset{K}, \Wea(\Sigma)) = \frac{1}{\domlp(\iset{K}, \Wea(\Sigma))} \ge \frac{1}{\dom(\Sigma)} \,,
\end{align*}
where we used the fact that $A_r^\Wea \subset \Wea(\Sigma)$ for any $r\le r_T$.

To lower bound $\eta_{\min}$, we use
\begin{align*}
\eta_{\min} =\min_{A\in \iset{K},|A|\ge 2}m(A,A^\Str) = \min_{A\in \iset{K},|A|\ge 2}m(A,A) =  \frac{1}{\max_{A\in \iset{K},|A|\ge 2}\domlp(A,A)}
\end{align*}
($A^\Str=A$ for strongly observable $\Sigma$), thus
\begin{align*}
\max_{A\in \iset{K},|A|\ge 2} \dom(A,A) \ge \frac{1}{\eta_{\min}} \,.
\end{align*}

For any $A\in \iset{K},|A|\ge 2$, let $\Sigma_A$ be the subgraph of $\Sigma$ on  $A$. We apply Lemma~\ref{lemma:dom-indp} on $\Sigma_A$ with the subset $W=A$. Then the lemma states that $A$ contains an independent set $U$ of size at least $\frac{\dom(A,A)}{50\log |A|}$. Since an independent set of $\Sigma_A$ is also an independent set of $\Sigma$, for each subset $A$ there exists an independent set of $\Sigma$  with size at least $ \frac{\dom(A,A)}{50\log |A|}$. So the independence number 
\begin{align*}
\indp(\Sigma)\ge \max_{A\in \iset{K},|A|\ge 2}  \frac{\dom(A,A)}{50\log |A|} \ge \frac{1}{50\log K}\max_{A\in \iset{K},|A|\ge 2} \dom(A,A) \ge \frac{1}{\eta_{\min}50\log K} \,,
\end{align*}
which indicates $\frac{1}{\eta_{\min}}\le \indp(\Sigma)50 \log K $.

\end{proof}

\subsection{Proof of Theorem~\ref{thm:ub2-alg2}}

\begin{proof}[Proof of Theorem~\ref{thm:ub2-alg2}]

Similarly to the proof of Theorem~\ref{thm:ub2-alg2}, we define high probability events
\[
U_r = \left\{ \forall i\in \iset{K}, |\htheta_{i,r}-\theta_i|\le g_{i,r}(\delta)\right\}\,.
\]
and upper bound the regret based on the fact that for all $r\ge 2$, $U_r$ holds. The rest of the proof will be based on upper bounding the number of round before all sub-optimal actions are eliminated.

Define $r_T=\max\{r \,:\, t_t<T, |A_r|\ge 2\}$, event 
\[
V_r=\left\{A_r^\Wea\ne \emptyset, \min_{i\in A_r^\Wea} n_i(r)< \min\{ \min_{i\in A_r^\Str} n_i(r), \gamma(r)\}\right\} 
\] and $V_r^c$ be its complement. 

For any $r\le r_T$ and any $i\in A_r$, $i\ne i_1(\theta)$, we have $2g_{i,r}(\delta)+ 2g_{i_1(\theta),r}(\delta)\ge \gap_i(\theta) \ge \gap_{\min}(\theta)$, where $\gap_{\min}(\theta)$ denotes $\gap_{i_2(\theta)}(\theta)$. From $g_{i,r}(\delta)=\sigma\sqrt{\frac{2\log(8K^2r^3/\delta)}{n_i(r)}}$ we get
\begin{align*}
\gap_{\min}(\theta)
\le 2\sigma\sqrt{2\log(8K^2r^3/\delta)} \left(\frac{1}{\sqrt{n_i(r)}} + \frac{1}{\sqrt{n_{i_1(\theta)}(r)}} \right) \le C_r\left(\min_{i\in A_r} n_i(r) \right)^{-1/2} \,,
\end{align*}
where $C_r=4\sigma \sqrt{6\log\frac{2Kr}{\delta}}$, and thus 
\begin{align*}
\min_{i\in A_r} n_i(r) \le \frac{C_r^2}{\gap_{\min}^2(\theta)} \,.
\addeq\label{eq:expl-needed}
\end{align*}

Then consider the regret:
\begin{align*}
R_T(\theta)
& \le \sum_{r=1}^{r_T} \I{V_r}\inprod{i_r}{\gap(\theta)} + \sum_{r=1}^{r_T} \I{V_r^c}\inprod{i_r}{\gap(\theta)} \\
& \le \sum_{r=1}^{r_V} \I{V_r} \norm{i_r}_1\gap_{\max}(\theta) + \sum_{r=1}^{r_W} \I{V_r^c} \norm{i_r}_1 \max_{i\in A_r} \gap_i(\theta)\,.
\addeq \label{eq:alg2-ub2-regret}
\end{align*} 
where $r_V=\max\{r\le r_T \,:\, V_r\}$ and $r_W=\max\{r\le r_T \,:\, V_r^c\}$. 

Since $\min_{i\in A_{r_V}^\Wea} n_i(r_V)<  \min_{i\in A_{r_V}^\Str} n_i(r_V)$ we have 
\begin{align*}
\min_{i\in A_{r_V}} n_i(r_V)=\min_{i\in A_{r_V}^\Wea} n_i(r_V) \ge \frac{1}{2\dom(\Sigma)} \sum_{s=1}^{r_V-1} \I{V_s}\norm{i_s}_1 - K^2 
\end{align*}
by applying Proposition~\ref{prop:weak-expl}. Then we can upper bound the first term in \eqref{eq:alg2-ub2-regret} by
\begin{align*}
\sum_{r=1}^{r_V} \I{V_r} \norm{i_r}_1 \le \frac{2\dom(\Sigma)C_{r_V}^2}{\gap_{\min}^2(\theta)} +2\dom(\Sigma)K^2+2K \,.
\addeq \label{eq:alg2-ub2-weak-expl}
\end{align*}
Regarding the second term in \eqref{eq:alg2-ub2-regret}, recall that for any $r\le r_T$ such that $t_r\ge T_0 = \frac{125\sigma^2}{D^2}K^3+10K^3$ and $V_r^c$ happens, $\min_{i\in A_r}n_i(r)\ge \gamma(r)\ge (\sigma\alpha_{r_T} t_r/D)^{2/3}\ge \left( \frac{\sigma t_r}{\dom(\Sigma)D} \right)^{2/3}$. Using the fact that $\max_{i\in A_r} \gap_i(\theta) \le \min \left\{ \gap_{\max}(\theta) , C_r\left( \min_{i\in A_r}n_i(r) \right)^{-1/2} \right\}$ gives
\begin{align*}
& \sum_{r=1}^{r_W} \I{V_r^c} \norm{i_r}_1 \max_{i\in A_r} \gap_i(\theta) \\
& \le \sum_{r=1}^{r_W} \I{V_r^c} \norm{i_r}_1 \min \left\{ \gap_{\max}(\theta),C_r ( \min_{i\in A_r}n_i(r) )^{-1/2} \right\} \\
& \le \sum_{r\ge 1: t_r<T_0} \norm{i_r}_1 \gap_{\max}(\theta) + \sum_{r\le r_W: t_r\ge T_0} \norm{i_r}_1 C_{r_W} \left(\frac{\sigma}{\dom(\Sigma)D}\right)^{-1/3} t_r^{-1/3} \\
& \le (T_0+2K)\gap_{\max}(\theta) +  C_{r_W} \left(\frac{\sigma }{\dom(\Sigma)D}\right)^{-1/3}  \sum_{r\le r_W: t_r\ge T_0} (t_{r+1}-t_r)(t_{r+1}-2K)^{-1/3} \\
& \le (T_0+2K)\gap_{\max}(\theta) +  C_{r_W}\left(\frac{\sigma }{\dom(\Sigma)D}\right)^{-1/3} \int_{T_0}^{t_{r_W+1}} (x-2K)^{-1/3} dx \\
& \le (T_0+2K)\gap_{\max}(\theta) +  C_{r_W} \left(\frac{\sigma }{\dom(\Sigma)D}\right)^{-1/3} \int_{T_0-2K}^{t_{r_W}} x^{-1/3} dx \\
& \le (T_0+2K)\gap_{\max}(\theta) +  \frac{3}{2}C_{r_W} \left(\frac{\sigma }{\dom(\Sigma)D}\right)^{-1/3} t_{r_W}^{2/3} \,.
\addeq \label{eq:alg2-regret-ub2-term2}
\end{align*}

Now we upper bound $t_{r_W}$. If $t_{r_W}\ge T_0$ then $\frac{C_{r_W}^2}{\gap_{\min}^2(\theta)} \ge \min_{i\in A_{r_W}}n_i(r_W) \ge \left( \frac{\sigma t_{r_W}}{\dom(\Sigma)D} \right)^{2/3}$. Hence
\begin{align*}
t_{r_W}^{2/3} \le \left(\frac{\sigma }{\dom(\Sigma)D}\right)^{-2/3} \frac{C_{r_W}^2}{\gap_{\min}^2(\theta)}+T_0^{2/3} 
\addeq\label{eq:alg2-ub2-trw} \,.
\end{align*}
Combining \eqref{eq:alg2-ub2-regret} \eqref{eq:alg2-ub2-weak-expl} \eqref{eq:alg2-regret-ub2-term2} and \eqref{eq:alg2-ub2-trw} with $C_{r_W} \le C_{r_T}$ gives
\begin{align*}
R_T(\theta)
& \le  \frac{1603\dom(\Sigma)D\sigma^2}{\gap_{\min}^2(\theta)} \left( \log\frac{2Kr_T}{\delta} \right)^{3/2} 
+ 14K^3D + \frac{125\sigma^2K^3}{D}\\
& +  15\left( \dom(\Sigma)D\sigma^2\right)^{1/3}\left(\frac{125\sigma^2}{D^2}+10 \right)K^2 \left( \log\frac{2Kr_T}{\delta} \right)^{1/2}
 \,. 
\addeq\label{eq:alg2-ub2-final}
\end{align*}
Applying $r_T\le T$ gives the result of Theorem~\ref{thm:ub2-alg2}. 

Note that using $r_T\le T$ here is only for simplicity, actually $r_T$ can be upper bounded by some constant by more careful analysis. This is because, according to Proposition~\ref{prop:weak-expl-2},  $\sum_{s=1}^{r_T}\I{V_s}\norm{i_s}_1 = O\left( t_{r_T}^{2/3} \right)$, and $t_{r_W} = O\left( (\log t_{r_T})^{3/2} \right)$, we have 
\begin{align*}
t_{r_T} \le t_{r_W} + \sum_{s=1}^{r_T}\I{V_s}\norm{i_s}_1 = O\left( t_{r_T}^{2/3} \right) + O\left( (\log t_{r_T})^{3/2} \right) \,,
\end{align*}
which mean $t_{r_T}$ must be upper bounded by some constant independent with $T$.

\end{proof}

\end{document}